\documentclass[twoside,11pt]{article}

\usepackage{jmlr2e}
\usepackage{amsmath}

\usepackage{algorithm}
\usepackage{algpseudocode}
\usepackage{subcaption}
\usepackage{mathtools}
\usepackage{float}
\usepackage{multirow}
\usepackage{color}
\usepackage{wrapfig}
\usepackage{caption}
\usepackage[dvipsnames]{xcolor}
\usepackage{colortbl}
\usepackage{booktabs}

\DeclareMathOperator*{\Ber}{Ber}
\DeclareMathOperator*{\ucb}{ucb}

\firstpageno{1}

\begin{document}

\title{Cost-Efficient Online Hyperparameter Optimization}
\author{\name Jingkang Wang$^{*1,2}$ \email wangjk@cs.toronto.edu
  \\ \name Mengye Ren$^{*1,2}$ \email mren@cs.toronto.edu
  \\
  \name Ilija Bogunovic$^{3}$ \email ilijab@ethz.ch
  \\
  \name Yuwen Xiong$^{1,2}$ \email yuwen@cs.toronto.edu
  \\
  \name Raquel Urtasun$^{1,2}$ \email urtasun@cs.toronto.edu \\
  \addr 
  University of Toronto$^1$, Uber ATG$^2$,
  ETH Zürich$^3$ \\
}

\maketitle

\begin{abstract}%
Recent work on hyperparameters optimization (HPO) has shown the possibility of training certain hyperparameters together with regular parameters. However, these online HPO algorithms still require running evaluation on a set of validation examples at each training step, steeply increasing the training cost. To decide when to query the validation loss, we model online HPO as a time-varying Bayesian optimization problem, on top of which we propose a novel \textit{costly feedback} setting to capture the concept of the query cost. Under this setting, standard algorithms are cost-inefficient as they evaluate on the validation set at every round.
In contrast, the cost-efficient GP-UCB algorithm proposed in this paper queries the unknown function only when the model is less confident about current decisions. We evaluate our proposed algorithm by tuning hyperparameters online for VGG and ResNet on CIFAR-10 and ImageNet100. Our proposed online HPO algorithm reaches human expert-level performance within a single run of the experiment, while incurring only modest computational overhead compared to regular training.
\end{abstract}


\section{Introduction}

Training deep neural networks involves a large number of hyperparameters, many of them found through repetitive trial-and-error. Often the results are very sensitive to the selection of hyperparameters and researchers typically follow classic ``cookbooks'' with many of their hyperparameters copied from the predecessor models. Hyperparameter optimization (HPO)~\citep{Hutter11,SnoekRSKSSPPA15,JamiesonT16} could potentially save much time from grid search procedures performed in nested loops; however, repetitive experiments on the order of hundreds are still required, which makes applying HPO prohibitively expensive in practice.

Recent advances in meta-optimization~\citep{Lorraine18,MacKay19} have shown that one can actually tune certain hyperparameters, e.g., data augmentation, dropout, example weighting, entirely online throughout the training of the main model parameters, by constantly inspecting at the validation loss. While this means that the training job only needs to be launched once, the cost of each training step rises sharply: to learn the hyperparameters, one has to take a gradient step from the reward signals obtained by evaluating on a separate set of validation examples to the hyperparameters. Depending on the size of the validation set, this can make the training time several times longer.

Ultimately, the goal is to design an online HPO algorithm that is efficient in terms of computation cost that arises from validation loss evaluations, i.e., whenever the algorithm queries for ``ground-truth" of validation loss.  
Towards building such an algorithm, we model the environment as a time-varying Bayesian Optimization (BO) problem with unknown time-varying reward/objective. In contrast with the standard BO setting, we require the agent to pay a certain cost to observe the reward every time it decides to query the unknown function. 
In this paper, we propose a time-varying cost-efficient GP-UCB~\citep{Srinivas09,Bogunovic16} algorithm, and provide kernel-based regret guarantees. Our algorithm makes use of a Gaussian process (GP) model to quantify uncertainty in the estimates and cost-efficient query rule. Based on this rule, the algorithm queries the unknown objective only in rounds when it is "less" confident in its decision.

We verify the effectiveness of our proposed algorithm empirically by tuning hyperparameters of large scale deep networks. First, we automatically adjust the tuning schedules of self-tuning networks~\citep{MacKay19} on CIFAR-10~\citep{cifar10}. Second, we optimize data augmentation parameters of a state-of-the-art unsupervised contrastive representation learning algorithm~\citep{Chen20simclr} on ImageNet100~\citep{imagenet}. We show that with modest computation overhead compared to regular training, one can achieve the same or better performance compared to baselines that have been extensively tuned by human experts.

To summarize, the contributions in this paper are as follows: (1) We introduce a novel \textit{costly feedback} setting for BO to model the evaluation cost for online HPO; (2) We propose a cost-efficient GP-UCB algorithm with kernel-based theoretical guarantees for time-varying BO with costly feedback; (3) We demonstrate the effectiveness of our method in tuning hyperparameters for modern deep networks in two challenging tasks.

\section{Related Work}
\paragraph{Bayesian optimization (BO):} BO is a popular framework for optimizing an unknown objective function from point queries that are assumed to be costly \citep{shahriari2015taking}. Besides the standard problem formulation \citep{Srinivas09}, different works have considered various  settings including contextual setting \citep{KrauseO11}, batch and parallel setting \citep{Desautels14}, optimization under budget constraints \citep{LamWW16} and robust optimization \citep{bogunovic2018adversarially}. %
Of particular interest to our work are time-varying BO methods \citep{Bogunovic16,Imamura20} that consider objectives that vary with time. In~\cite{Bogunovic16}, the authors consider time variations that follow a simple Markov model. Their proposed algorithm TV-GP-UCB comes with the model-based forgetting mechanism and attains regret bounds that jointly depend on the time horizon and rate of function variation. In this work, we consider the same time-varying model %
but in the setting of sparse observations. That is, while  previous works have focused on the setting where observations are received at every time step, our goal is to optimize an unknown time-varying objective subject to a specified budget constraint. %

\paragraph{Hyperparameter Optimization (HPO):} There are three mainstream frameworks in HPO: model-free, model-based and gradient-based approaches.  Specifically, grid search, random search~\citep{bergstra2012random} and its extensions (e.g., successive halving~\citep{JamiesonT16} and Hyperband~\citep{LiJDRT17}) are standard model-free HPO methods that ignore the structure of the problem and model. 
In contrast, BO~\citep{Hutter11,BergstraBBK11,SnoekLA12,SnoekRSKSSPPA15} is a common model-based approach that aims to model the conditional probability of performance given the hyperparameters and a dataset. Other assumptions such as learning curve behavior~\citep{SwerskySA14,KleinFSH17,Nguyen19bayesian} and computational cost~\citep{KleinFBHH17} are taken into account in BO to avoid learning from scratch every time. However, both model-free and model-based approaches involve repetitive trial-and-error processes that are very expensive in practice. An alternative gradient-based solution is to cast HPO as a bilevel optimization problem and take gradients with respect to the hyperparameters. Since unrolling the whole learning trajectories is prohibitively expensive, researchers usually consider a biased several-step look-ahead approximation~\citep{Domke12,LuketinaRBG16,FranceschiFSGP18} or the implicit function theorem~\citep{larsen1996design,Pedregosa16,Lorraine2019,Bertrand20implicit} to obtain the gradients. To further improve the efficiency of HPO, recent works~\citep{Lorraine18,MacKay19} utilize hypernetworks~\citep{hypernetwork} to approximate the inner optimization loop and a held-out validation set to collect reward signals. Instead of constantly evaluating the validation loss, which brings remarkable computation overhead, this work propose a cost-efficient evaluation rule, and emprirically demonstrates it effectiveness in the real HPO tasks.

\section{Bayesian Optimization with Costly Feedback}
Recent online HPO algorithms require to obtain reward signals constantly by evaluating the metrics such as performance gain on the validation set, drastically increasing the training cost. Motivated by this observation, 
we model HPO as a time-varying Bayesian optimization (BO) problem where the unknown function is treated as the reward signals obtained through subsequent evaluations. To capture the evaluation cost induced by observing the rewards, we introduce a novel \textit{costly feedback} setting that allows the agent to decide, at every round, whether to receive the observation. In turn, the agent is required to pay a certain cost whenever it receives feedback. %

\subsection{Problem Setup}

We aim to sequentially optimize an unknown objective function $f_t(x): \mathcal{D} \times
\mathcal{T} \rightarrow \mathbb{R}$ defined on composite space $\mathcal{X} = \mathcal{D} \times
\mathcal{T}$, where  $\mathcal{D}$ is the finite input domain%
and $\mathcal{T} = \{1, 2, \dots, T\}$ represents the time domain. At each round $1 \leq t \leq T, t \in \mathbb N $, the agent decides upon a data point $x_t$.
After selecting the point $x_t$, it has the option to decide whether to receive the feedback by interacting with $f_t$. If the agent chooses to observe the
feedback, then it receives a noisy observation $y_t = f_t(x) + z_t$, where $z_t$ is assumed to be independently sampled from a Gaussian distribution $\mathcal{N}(0, \sigma^2)$. %
Let $\mathcal{S}_{t}^n = \{(x_{h(i)}, y_{h(i)}, i)\}_{i=1}^{n}$ denote the data obtained through $n$ observations till round $t$. %
Here $h(i)$ denotes the mapping function that records the round at which the agent observes the $i$-th data point. The agent then chooses its next point $x_{t+1}$ based on the previously collected data $\mathcal{S}_{t}^{n}$.
If the agent decides to query the unknown function at round $t$, then it needs to pay a cost $c_t$.
We let $c_t = 1$ if the feedback $y_t$ is received at round $t$, and  otherwise $c_t = 0$. We define the \textit{cumulative cost} as $C_T = \sum_{t=1}^T c_t$ that records the total number of queries within $T$ rounds. %

\paragraph{Connection to online HPO:} In this paper, online HPO is modeled as the time-varying BO with costly feedback -- hyperparameter configurations $x_t$ are selected sequentially to maximize the reward signals $f_t(x)$. Since evaluation on the validation set is expensive, the agent is allowed to skip the evaluation or pay a certain cost $c_t$ to see the reward.

To measure the performance of algorithms, the regret for $t$-th round is defined as $r_t = \max_{x \in \mathcal{D}} f_t(x) - f_t(x_t)$. The
\textit{cumulative regret} $R_T$ is the sum of instantaneous regrets  $R_T = \sum_{t=1}^{T} r_t$. Furthermore, we define the \textit{cumulative loss} as $L_T = R_T + C_T$ to balance between regret and cost.
We note that if an agent queries the unknown objective function at each time step, we recover the standard time-varying BO setting \cite{Bogunovic16}. %

\section{Cost-Efficient GP-UCB Algorithm}
Since the reward signals for online HPO vary as the learning proceeds, and similar hyperparameter configurations lead to similar performance, we consider using a time-varying Gaussian process (GP) to model the objective function $f_t$. Consequently, we start this section by studying the behavior of time-varying GP models with sparse observations. We then propose a novel cost-efficient algorithm and study its theoretical performance. We defer all the analysis to the supplementary material.

\subsection{Time-varying Gaussian Process Model}
\label{sec:gp_costly}

We use a spatio-temporal GP to model the underlying function $f_t$. The smoothness properties of $f_t$ are depicted by a composite kernel function~\citep{Bogunovic16,Imamura20}: $k=k_{\mathrm {space}} \otimes k_{\mathrm{time}}$, where $k_{\mathrm{space}} : \mathcal{D}
\times \mathcal{D} \rightarrow \mathbb R^{+}$ and $k_{\mathrm{time}} : \mathcal{T} \times
\mathcal{T} \rightarrow \mathbb R^{+}$ are spatial and temporal kernels; $(k_{\mathrm {space}} \otimes k_{\mathrm{time}})((x,t), (x',t')) := k_{\mathrm {space}}(x,x')\cdot k_{\mathrm{time}}(t,t')$.
Following~\cite{Bogunovic16}, we consider the following transition relation of the underlying function: $f_{t+1}(x) = \sqrt{1 - \epsilon} f_{t}(x) + \sqrt{\epsilon} g_{t+1}(x)$, where $g_t$ is sampled from a GP with a mean function $\mu$ and a kernel function $k$, i.e., $g_t
\sim \mathcal{G} \mathcal{P}(\mu, k)$\footnote{Without loss of generality, as in~\cite{Srinivas09}, we assume $\mu = 0$ for GPs not conditioned on data.}; $\epsilon \in[0, 1]$ is the forgetting rate that constrains the variation of the function.
This particular Markov model maps to the following exponential temporal kernel as shown in~\cite{Bogunovic16}: %
$
k_{\mathrm{time}}\left(\tau, \tau^{\prime}\right)=(1-\epsilon)^{\frac{\left|\tau-\tau^{\prime}\right|}{2}}.
$

Given observed data points $\mathbf x_t^n = \{x_{h(1)}, \dots, x_{h(n)}\}$ and $\mathbf y_t^n =
\{y_{h(1)}, \dots, y_{h(n)}\}$, the posterior over $f$ is a GP with mean $\tilde{\mu}_t(x)$,
covariance $k_t (x, x^{\prime})$ and variance $\tilde{\sigma}_t^2 (x) = k_t (x, x)$:
\begin{small}
\begin{align}
\tilde{\mu}_{t}(x) &=\widetilde{\mathbf{k}}_{t}^n(x)^{T}\left(\widetilde{\mathbf{K}}_{t}^n+\sigma^{2} \mathbf{I}_{t}^n\right)^{-1} \mathbf{y}_{t}^n, \label{eq:posterior_mean}\\
k_t(x, x^{\prime}) &=k(x, x^{\prime})-\widetilde{\mathbf{k}}_{t}^n(x)^{T}\left(\widetilde{\mathbf{K}}_{t}^n+\sigma^{2} \mathbf{I}_{t}^n\right)^{-1} \widetilde{\mathbf{k}}_{t}^n(x^{\prime}), %
\label{eq:posterior_var}
\end{align}
\vspace{-0.05in}
\end{small}
where $\widetilde{\mathbf{K}}_{t}^n=\mathbf{K}_{t}^n \circ \mathbf{D}_{t}^n$, with $\mathbf{D}_{t}^n
=\left[(1-\epsilon)^{|h(i)-h(j)| / 2}\right]_{i, j=1}^{n}$, and $
\tilde{\mathbf{k}}_{t}^n(x)=\mathbf{k}_{t}^n(x)  \circ  \mathbf{d}_{t}^n $ with $ \mathbf{d}_{t}^n
=$ $\left[(1-\epsilon)^{(t+1-h(i)) / 2}\right]_{i=1}^{t}$. Here  $\mathbf{K}_{t}^n=$ $\left[k\left(x_i,
x_j\right)\right]^{n}_{i,j=1}$,  $\mathbf{k}_{t}^n(x)=\left[k\left(x_{i},
x\right)\right]_{i=1}^{n}$,  $\circ$ is the Hadamard product, and $\mathbf{I}_{t}^n$ is the $n
\times n$ identity matrix.

One key challenge in BO is to balance the exploration-exploitation tradeoffs rigorously. Various works (e.g.,~\cite{Srinivas09,Bogunovic16}) have focused on the \textit{upper-confidence bound} (UCB) rule that selects points by maximizing a linear combination of posterior mean and variance. 
In this paper, our focus is also on the UCB rule due to its strong theoretical guarantees and empirical performance. In what follows, we let $\ucb_t(x) := \tilde{\mu}_t(x) + \beta_{t+1}^{1/2}\tilde{\sigma}_t(x)$, where $\lbrace \beta_t\rbrace_{t=1}^T$, each $\beta_t \in \mathbb{R_{+}}$, is a non-decreasing sequence of exploration parameters, selected as~\cite{Bogunovic16}, such that (w.h.p.) $\ucb_t(x)$ is a valid upper confidence bound on $f_t(x)$ for every $x$ and $t$.

\subsection{Strategies for Querying Feedback}
\label{sec:strategies}
In BO with costly feedback, there are two decisions for the agent to make at every round: 1) where to evaluate the unknown objective; 2) whether to receive the feedback (and consequently incur cost). In this section, we focus on the latter problem. 
To minimize regret while reducing the query cost, we %
propose two strategies for the query allocation.

\vspace{-0.05in}
\paragraph{Querying with Bernoulli Sampling Schedule:}
\label{sec:uniform_tvgp}
A simple and model-agnostic method for our problem is to query the objective function according to a fixed probability. Specifically, we query the feedback based on a Bernoulli random variable $\Ber(B/T)$, which guarantees that the algorithm receives in expectation $B$ observations of $y_t$.

\paragraph{Leveraging Uncertainty Information from GP:} 
Although the Bernoulli strategy defined above can reduce the query cost, it does not leverage any knowledge from the GP model. As a consequence, in practice, there is often a significant performance loss compared to standard algorithms with full observations. To overcome this difficulty, we propose a cost-efficient query rule that automatically assesses the uncertainty of the current decision for time-varying GP-UCB~\citep{Bogunovic16} and its variants. 

On a high level, the agent aims to maintain informative queries but skip uninformative ones to save the query cost. %
Hence, we consider the following cost-efficient query rule: the agent only queries the feedback when the following condition satisfies: $$\mathbb P \left(\hat{y}_t({x_t}) > \hat{y}_t({x}) \right) < \kappa, \ \exists x \in \mathcal{D}\setminus \{x_t\},$$
where $\hat{y}_t(x)$ is the posterior predictive estimation of the unknown objective at point $x$, which is a random variable distributed according to $\mathcal{N}\left(\tilde{\mu}_t(x), \tilde\sigma^2_t(x)\right)$. The analytic solution to $\mathbb P \left(\hat{y}_t({x_t}) > \hat{y}_t({x}) \right) < \kappa$, as well as the analysis of the choices of $\kappa$ are provided in Appendix~\ref{supp:kappa_choice}. This rule explicitly encourages receiving feedback when the agent is less confident in distinguishing the best candidate. The intuition behind this rule is to skip querying $f_t$ when the selected $x_t$ leads to ``small'' regret.

\subsection{Cost-Efficient GP-UCB Algorithm}
\label{sec: ce-gp-ucb}
We integrate the two strategies for querying feedback in time-varying GP-UCB model and obtain two algorithms for our costly feedback setting. Specifically, we give the cost-efficient algorithm  (CE-GP-UCB) that leverages the uncertainty information from GP in Algorithm~\ref{alg:ca-gp-ucb}. Our algorithm selects the points with largest UCB at every round (Line~\ref{alg:line3}) but only observes the feedback $y_t$ when the when the model is most uncertain
(Line~\ref{alg:line8}). Note that the query strategy can be replaced by Bernoulli sampling strategy to consume the budget $B$ as introduced in Sec~\ref{sec:strategies}. %

When no feedback is received (Line~\ref{alg:line10}), there is no update of the GP model; however, due to the time-varying nature, the model will be less confident about the estimation for the candidates thus the posterior variance will increase. 

\paragraph{Remark:} Note that if the candidates are close to each other (e.g., quantized  candidates  for  continuous  variables),  the  model  will  never  be  confident  in  its current decision since the UCB of best and second-best candidates are very close.  In other words, the cost-efficient query rule (Line 4 in Algorithm 1) will continuously be activated. However it is less informative to query the feedback for  two close candidates within one mode. Therefore, in practice, we  first find the local optima of the UCB function and their corresponding data points, then deploy the cost-efficient query rule for these local optima.

\begin{algorithm}[t]
\caption{CE-GP-UCB}\label{alg:ca-gp-ucb}
\begin{small}
\begin{algorithmic}[1]
\Require Input Domain $\mathcal{D}$, total rounds $T$, GP prior $\left(\tilde{\mu}_{0}, \tilde{\sigma}_{0}, k\right)$, forgetting rate $\epsilon$, 
user-specific confidence threshold $\kappa$
\State Initialize observation set $\mathcal{S}_{0}^{0} = \emptyset$, number of interactions $n = 0$.
\For {$t = 1, 2, \cdots, T$} 
\State Choose $x_{t}=\arg \max _{x \in \mathcal{D}} \tilde{\mu}_{t-1}(x)+\sqrt{\beta_{t}} \tilde{\sigma}_{t-1}(x)$ \Comment{Select points according to UCB rule} \label{alg:line3}
\If {$\mathbb P \left(\hat{y}_t({x_t}) >  \hat{y}_t({x})\right) < \kappa $, $\exists x \in \mathcal{D} \setminus \{x_t\}$}
\Comment{Cost-efficient query strategy} \label{alg:line8}
\State Receive feedback $y_{t}=f_{t}\left(x_{t}\right)+z_{t}$, $n\leftarrow n+1$ \label{alg:line5} {\Comment{Interact with $f_t$ and receive feedback}}
\State Add $(x_t, y_t, t)$ to observation set:
$\mathcal{S}_{t}^{n+1} = \mathcal{S}_{t-1}^{n} \cup {(x_t, y_t, t)} $
\State Perform Bayesian update to obtain $\tilde{\mu}_{t}$ and $\tilde{\sigma}_{t}$ \label{alg:line7}
\Else {\ $\tilde{\mu}_{t} = \tilde{\mu}_{t-1}, \tilde{\sigma}_{t} = \tilde{\sigma}_{t-1}$, $\mathcal{S}_{t}^n = \mathcal{S}_{t-1}^{n}$} \Comment{Do not observe feedback} \label{alg:line10}
\EndIf
\EndFor
\end{algorithmic}
\end{small}
\end{algorithm}

\section{Experiments}
\label{sec:exp}
We thoroughly evaluate the proposed algorithm on both synthetic data and practical online hyperparameter optimization problems: (1) tuning schedules for self-tuning networks (STN~\cite{MacKay19}); (2) data augmentations for unsupervised representation learning~\citep{Chen20simclr}.
Extensive experiments show that cost-efficient query rule leads to substantial improvements over simple Bernoulli strategy. %
We leave extra experimental details in Appendix~\ref{supp:exp}.

\vspace{-0.05in}
\subsection{Evaluation on Synthetic Data}

We follow the same synthetic setting as previous study~\cite{Bogunovic16}. Specifically, we used a one-dimensional input domain $\mathcal{D} = [0, 1]$
and quantized it into 1,000 uniformly divided points. We generated the time-varying objective functions according to the following Markov model $f_{t+1}(x)=\sqrt{1-\epsilon} f_{t}(x)+\sqrt{\epsilon} g_{t+1}(x)$ under different forgetting rate $\epsilon$, where $g_t(x) \sim \mathcal{GP}(0, k)$. We use Matérn3/2 kernel for GP models and set sampling noise variance $\sigma^2 = 0.01$.  The time horizon $T$ is set as 500. %

Figure~\ref{fig:bo_synthetic} shows the trade-off curves between average regret $R_T/T$ and query cost $C_T$ for different algorithms, where the performance is averaged over 50 independent trials. Specifically, we consider the TV-GP-UCB~\citep{Bogunovic16}, R-GP-UCB~\citep{Bogunovic16}, and its variants with other popular acquisition functions (PI: probability of improvement; EI: expected improvement) with Bernoulli query policy.
We observe that our method suffers from a minor regret loss when using 50\% query cost ($C_T = 250$), whereas other baselines with Bernoulli strategy lead to a larger performance loss. We then visualize the final GP models in Figure~\ref{fig:tvgp_synthetic} ($\epsilon = 0.03$). In particular, CE-GP-UCB maintains a similar posterior mean and variance for the unknown function compared with TV-GP-UCB with full observations and skips the queries when chosen candidates are close to the optimal ones. By contrast, TV-GP-UCB with Bernoulli strategy skips the queries even when the selected points are far from the optimal ones, thus leading a larger regret.

\begin{figure}[t]
    \centering
    \includegraphics[width=\textwidth]{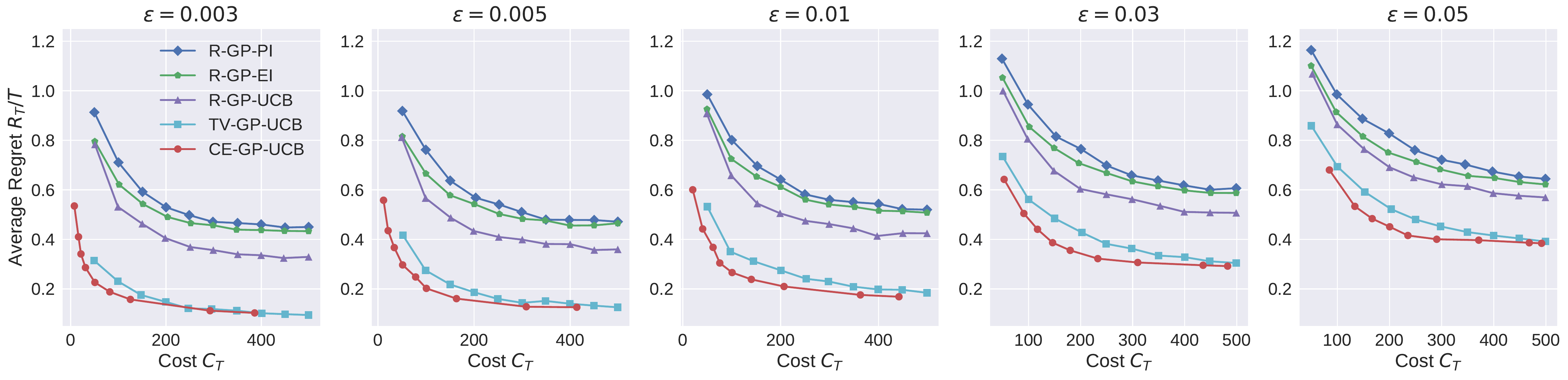}
    \vspace{-0.25in}
    \caption{Comparison of trade-off between average regret $R_T/T$ and query cost $C_T$.}
    \label{fig:bo_synthetic}
\end{figure}

\begin{figure}[t]
\centering
\begin{subfigure}[b]{.24\textwidth}
    \centering
    \includegraphics[width=\textwidth]{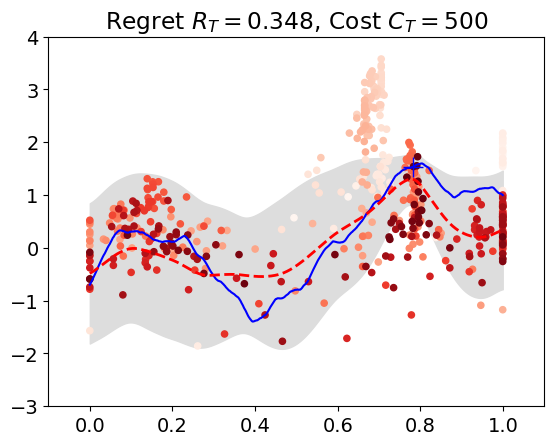}
    \vspace{-5mm}
    \caption{\texttt{TV-GP-UCB}}
    \label{fig:tv_gp_ucb}
\end{subfigure}
\begin{subfigure}[b]{.24\textwidth}
    \centering
    \includegraphics[width=\textwidth]{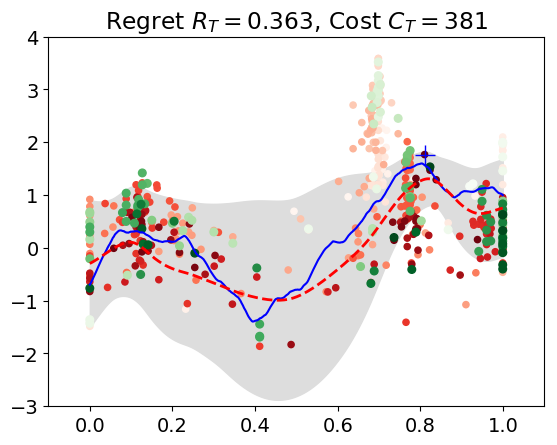}
    \vspace{-5mm}
    \caption{\texttt{TV-GP-UCB} Ber(0.8)}
    \label{fig:ca_gp_ucb_0.9}
\end{subfigure}
\begin{subfigure}[b]{.24\textwidth}
    \centering
    \includegraphics[width=\textwidth]{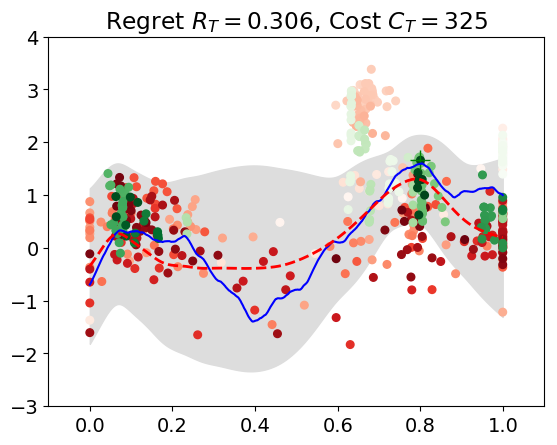}
    \vspace{-5mm}
    \caption{ours ($\kappa=0.90$)}
    \label{fig:ca_gp_ucb_0.90}
\end{subfigure}
\begin{subfigure}[b]{.24\textwidth}
    \centering
    \includegraphics[width=\textwidth]{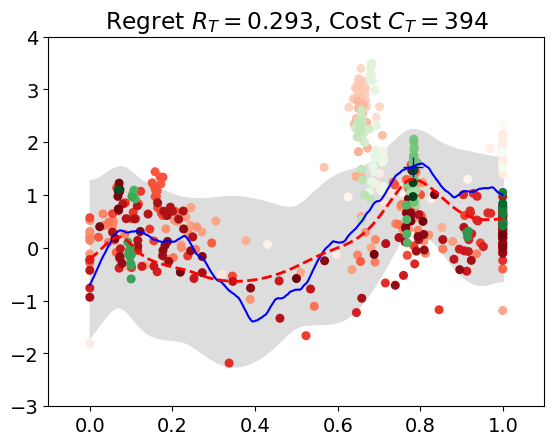}
    \vspace{-5mm}
    \caption{ours ($\kappa=0.95$)}
    \label{fig:ca_gp_ucb_*}
\end{subfigure}
\vspace{-0.05in}
\caption{
Visualizations of GP models. The \textcolor{purple}{\textbf{purple points}} and \textcolor{ForestGreen}{\textbf{green points}} denote the agents' choices to receive the feedback or not. Data points with darker color mean that they are visited more recently. The \textcolor{Blue}{\textbf{blue line}} (\textcolor{Blue}{\textbf{---}}) and \textcolor{red}{\textbf{dashed red line}} (\textcolor{red}{\textbf{-}}\textcolor{red}{\textbf{-}}\textcolor{red}{\textbf{-}}) indicate the ground truth of unknown objective function and the posterior mean. The \textcolor{gray}{\textbf{gray area}} denotes the confidence area. %
}
\label{fig:tvgp_synthetic}
\end{figure}

\subsection{Self-Tuning Networks} 
To further study the effectiveness of our algorithm in real applications, we first evaluate the proposed method in adjusting the tuning schedule for self-tuning networks~\citep{Lorraine18} (STN). STN is an online HPO algorithm that uses a hypernetwork~\citep{hypernetwork} to approximate the inner loop of bilevel optimization. Standard STN tunes hyperparameters (e.g., dropout, weight-decay and data augmentation) by alternating the following two steps (i.e., train/val=1:1): (1) (\textit{training phase}) train one mini-batch data on training set; (2) (\textit{tuning phase}) evaluate one mini-batch data on held-out validation set then backpropagate gradients through the hypernetwork to update hyperparameters. However, we found that STN is sensitive to different tuning schedules and standard ``dense  tuning'' is expensive and sub-optimal.

\begin{wrapfigure}{r}{0.53\textwidth}
\vspace{-8mm}
\begin{minipage}{0.53\textwidth}
\begin{table}[H]
\centering
\caption{Comparison with conventional MAB algorithms for STN (VGG16 on CIFAR-10). %
}
\label{tab:stn_vs_mab}
\resizebox{\textwidth}{!}{
\begin{tabular}{@{}llccccc@{}}
\toprule
\multicolumn{2}{l}{} & $\ell_{val}$ & Acc$_{val}$ & $\ell_{test}$ & Acc$_{test}$ & $T_{total}$ \\ \midrule
\multicolumn{2}{c}{Grid Search} & 0.421 & 0.887 & 0.444 & 0.879 & 11.69 $\times$ \\ \midrule
\multirow{5}{*}{\ Ber(0.1)} & EXP3.R & 0.466 & 0.841 & 0.479 & 0.835 & 0.93 $\times$ \\
 & $\epsilon$-greedy & 0.419 & 0.864 & 0.438 & 0.859 & 0.92 $\times$ \\
 & Softmax & 0.433 & 0.854 & 0.459 & 0.846 & 0.93 $\times$ \\
 & UCB & 0.447 & 0.869 & 0.451 & 0.867 & 0.94 $\times$ \\ 
 & GP-UCB & 0.422 & 0.877 & 0.449 & 0.868 & 0.96 $\times$ \\  \midrule
\multirow{5}{*}{\ Ber(0.2)}  & EXP3.R & 0.446 & 0.853 & 0.455 & 0.847 & 1.13 $\times$ \\
 & $\epsilon$-greedy & 0.401 & 0.881 & 0.443 & 0.873 & 1.13 $\times$ \\
 & Softmax & 0.449 & 0.862 & 0.480 & 0.852 & 1.13 $\times$ \\

 & UCB & 0.398 & 0.878 & 0.421 & 0.871 & 1.13 $\times$ \\ 
 & GP-UCB & 0.413 & 0.880 & 0.439 & 0.870 & 1.17 $\times$ \\
 \midrule
\multicolumn{2}{l}{CE-GP-UCB ($\kappa = 0.7$)} & 0.390 & 0.888 & 0.428 & 0.881 & 1.10 $\times$ \\
\multicolumn{2}{l}{CE-GP-UCB ($\kappa = 0.8$)} & 0.373 & 0.892 & 0.404 & 0.881 & 1.21 $\times$ \\ \bottomrule
\end{tabular}
}
\vspace{-3mm}
\end{table}
\end{minipage}
\end{wrapfigure}

As a consequence, we model STN as a two-armed bandits: ``\textit{training only}'' and ``\textit{tuning + training}''. Since tuning one mini-batch data is biased, we define the unknown objective function as the accuracy gain on the whole validation dataset, which is expensive to obtain thus leading a large query cost. 
We evaluate the proposed approach and other TV bandit algorithms for VGG16 on CIFAR-10. Specifically, we randomly choose 10,000 training images as the validation set, and tune layerwise dropout and data augmentation parameters, following~\cite{Lorraine18}. The network is trained 150 epochs with an initial learning rate 0.1. The learning rate is decayed at 60, 100 and 120 epochs with a decay rate of 0.1.

\begin{figure}[t]
\centering
\begin{subfigure}[b]{.29\textwidth}
    \centering
    \includegraphics[width=\textwidth]{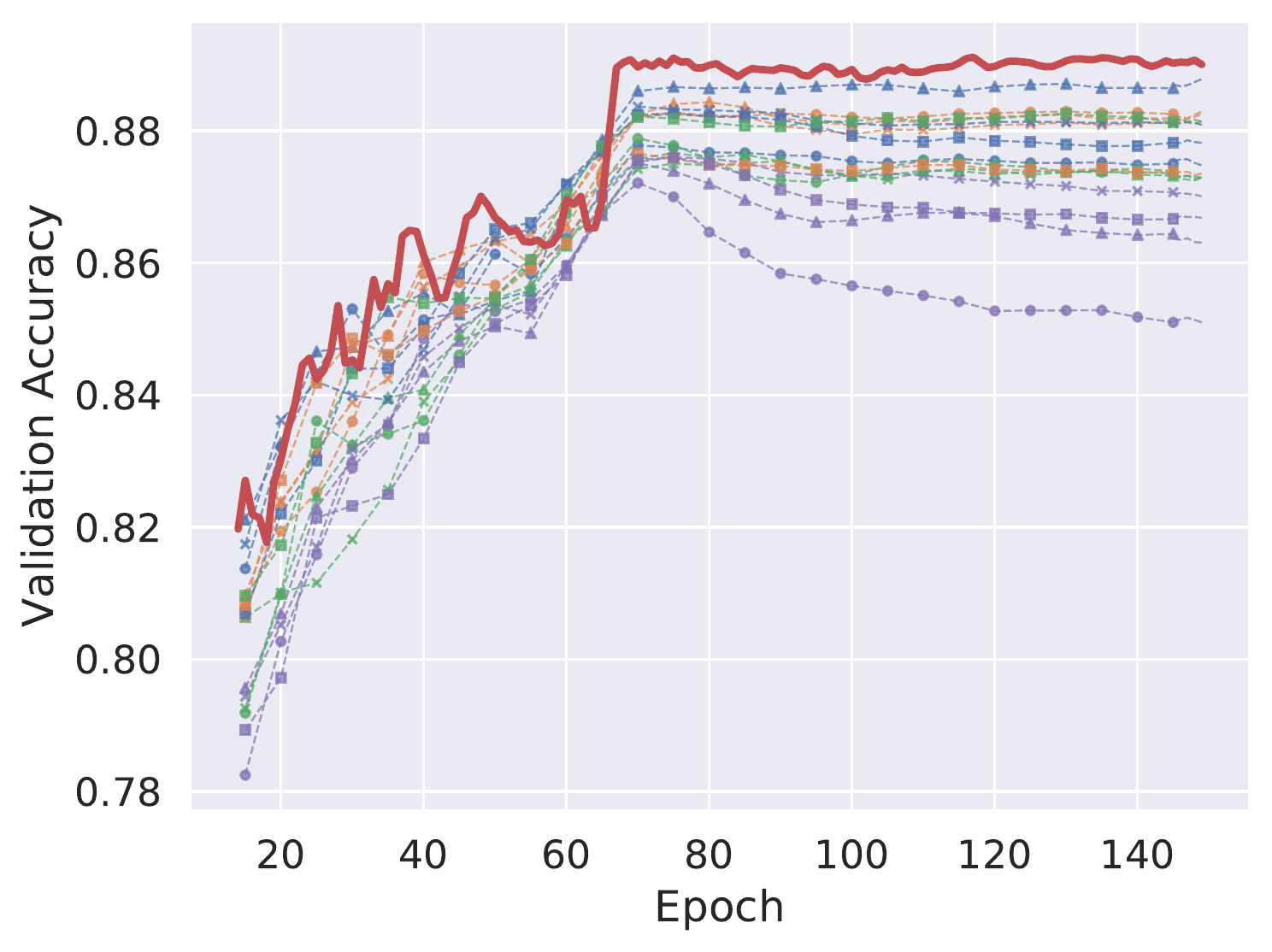}
    \vspace{-5mm}
    \label{fig:stn_val_acc}
\end{subfigure}
\hspace{-0.05in}
\begin{subfigure}[b]{.39\textwidth}
    \centering
    \includegraphics[width=\textwidth]{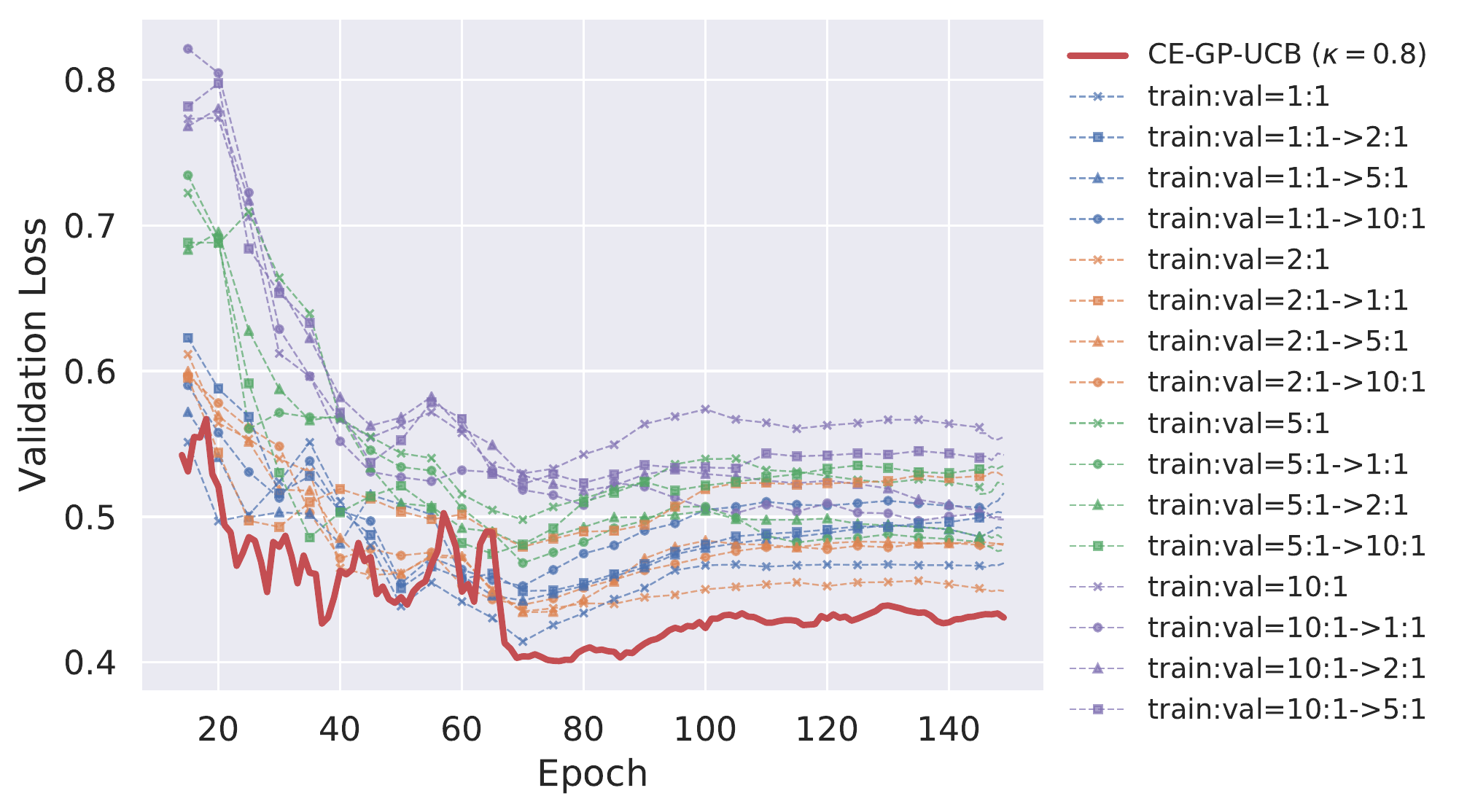}
    \vspace{-5mm}
    \label{fig:stn_val_loss}
\end{subfigure}
\begin{subfigure}[b]{.295\textwidth}
    \centering
    \includegraphics[width=\textwidth]{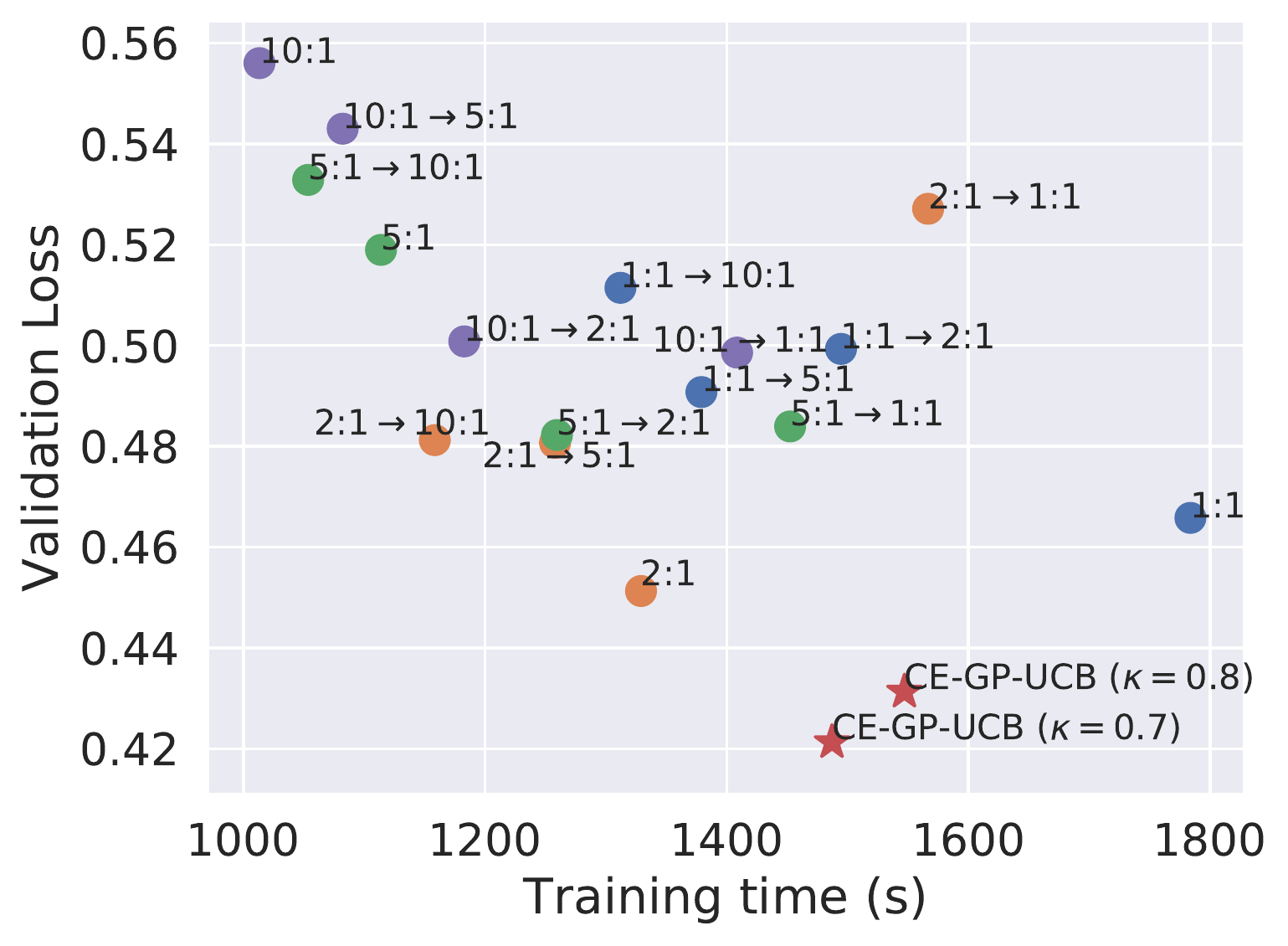}
    \vspace{-5.5mm}
    \label{fig:time_vs_loss}
\end{subfigure}
\vspace{-0.05in}
\caption{(a) Validation accuracy and \& (b)  validation loss for STN with different tuning schedules. (c) Trade-offs between validation loss and corresponding training time (excluding querying cost).}
\label{fig:stn_learning_curves}
\end{figure}

We first consider two baselines including static schedules~\citep{MacKay19} and a simple switching heuristics (e.g., train:val=1:1$\rightarrow$2:1) that occurs at 100 epochs in which the learning rate is decayed. Since smaller learning rate is easier to cause overfitting, we expect to find a better schedule by also decaying the tuning frequency at the late training stage. The dashed lines in Figure~\ref{fig:stn_learning_curves} (a) \& (b) show that the performance of STN is sensitive to varied tuning schedules. %
As shown in Figure~\ref{fig:stn_learning_curves} (a) \& (b), our approach leads to the best performance in terms of validation accuracy and loss with only a single run of the experiment. In Figure~\ref{fig:stn_learning_curves} (c), CE-GP-UCB successfully finds an efficient tuning schedule online that saves around 20\% training cost meanwhile achieving lower validation loss. 
Table~\ref{tab:stn_vs_mab} gives a quantitative comparison between our algorithm and other baselines including grid search and TV bandits algorithms with Bernoulli strategy. After taking account of the cost for evaluation on the validation set, we find that the GP-UCB algorithm with cost-efficient query rule outperforms other baselines on both validation and testing set with modest computation overhead.

\subsection{Unsupervised Representation Learning}
We evaluate our approach on tuning data augmentations in unsupervised contrastive learning. The goal of unsupervised learning is to learn meaningful representations directly from unlabeled data since acquiring annotations is expensive.
SimCLR~\citep{Chen20simclr} is the state-of-the-art approach that learns representations by maximizing agreement between differently augmented views of the same data example.
It is shown in~\cite{Chen20simclr} that the types of data augmentations are critical in learning meaningful representations so researchers conducted extensive ablation study in data augmentations. 
By contrast, we aim to apply CE-GP-UCB in tuning the probability of randomly applying eight common data augmentations including \textit{cropping}, \textit{color distortion}, \textit{cutout}, \textit{flipping horizontally and vertically}, \textit{rotation}, \textit{Gaussian blur} and \textit{gray scale} with one single run. We define the accuracy gain (\%) on validation set as the costly feedback. The initial probability for applying each data augmentation is set 0.5 and the tuning range is $[0, 1.0]$. The forgetting rate $\epsilon$ and $\beta_t$ are set as 0.01 and 1.0, respectively. We leave the implementation details of SimCLR in Appendix~\ref{supp:exp}.

\begin{wrapfigure}{r}{0.55\textwidth}
\vspace{-0.28in}
\begin{minipage}{0.55\textwidth}
\begin{table}[H]
\centering
\caption{Linear readout performance (ImageNet100 top-1 accuracy) of ResNet50 with different data augmentations.} %
\vspace{-0.1in}
\label{tab:simclr}
\resizebox{\textwidth}{!}{
\begin{tabular}{lcccc} \toprule
 & Top1 (R10) & Top1 (R100) & Time \\ \midrule
Baseline & 70.91 & 73.20 & 1.00 $\times$ \\
TV-GP-UCB (full) & 75.14 & 77.95 & 1.97 $\times$ \\
TV-GP-UCB Ber(0.6) & 72.15 & 75.31 & 1.64 $\times$ \\ %
CE-GP-UCB ($\kappa=0.9$) & 74.80 & 77.56 & 1.88 $\times$ \\
CE-GP-UCB ($\kappa=0.8$) & 74.77 & 77.62 & 1.71 $\times$ \\
CE-GP-UCB ($\kappa=0.7$) & 74.58 & 77.27 & 1.61 $\times$ \\ \midrule
Human expert~\citep{Chen20simclr} & 75.00 & 77.99 & - \\\bottomrule
\end{tabular}
}
\end{table}
\end{minipage}
\vspace{-0.1in}
\end{wrapfigure}

Table~\ref{tab:simclr} shows the linear readout performance for different models, where the baseline is SimCLR with fixed initial probability (0.5) for eight data augmentations. As shown in Table~\ref{tab:simclr}, our method reaches human expert-level performance with one single run of experiment and surpass baseline (no HPO) by a large margin. More importantly, proposed CE-GP-UCB is able to reduce 40\% query cost for validation with minor performance loss compared to TV-GP-UCB with full observations; however, TV-GP-UCB with Bernoulli strategy results in a large performance loss. It again verifies proposed cost-efficient query rule is able to successfully maintain the most informative queries and provide an alternative BO solution when resources are limited.

\vspace{-0.1in}
\section{Conclusion}
Online HPO typically requires constantly evaluating on the validation set and taking gradient steps \textit{w.r.t.} hyperparameters, resulting in drastically higher training cost. In this paper, we propose a novel \textit{costly feedback} Bayesian optimization (BO) setting to model the computation cost for querying the reward signals from the validation set. To keep most informative queries and skip less informative ones, we introduce a cost-efficient GP-UCB algorithm that automatically assesses the uncertainty of current GP model.  We further verify the effectiveness of our proposed approach with extensive experiments on both synthetic data and large-scale real world online HPO for deep neural networks.

\section*{Acknowledgement}
We thank Renjie Liao, Sivabalan Manivasagam and James Tu for their thoughtful discussions and useful feedback.

\bibliography{reference}

\newpage
\appendix

\setcounter{section}{0}
\setcounter{figure}{0}
\setcounter{table}{0}
\setcounter{theorem}{0}

\makeatletter 
\renewcommand{\thefigure}{A\arabic{figure}}
\setcounter{table}{0}
\renewcommand{\thetable}{A\arabic{table}}
\setcounter{algorithm}{0}
\renewcommand{\thealgorithm}{A\arabic{algorithm}}

\section{Choices of Confidence threshold $\kappa$ in CE-GP-UCB}
\label{supp:kappa_choice}
In what follows, we first give an important property induced by no overlapping between the confidence bounds of best candidate $x_t$ and the other points $x \in \mathcal{D} \setminus \{x_t\}$. 

\begin{lemma}
Let $\delta \in (0, 1)$ and $x_t = \arg \max _{x \in \mathcal{D}}
\tilde{\mu}_{t-1}(x)+\sqrt{\beta_{t}} \tilde{\sigma}_{t-1}(x)$, if $\mathrm{ucb}(x) \leq \mathrm{lcb}(x_t), \forall x \in \mathcal{D} \setminus \{x_t\}$, then picking $x_t$ at round $t$
produces no regret with at least probability $1 - \delta$. %
\label{lem:free_no_regret}
\end{lemma}

\begin{proof}[Proof of Lemma~\ref{lem:free_no_regret}]
From previous literature, we know that, with high probability
\begin{align}
    \left|f_t(x) - \tilde{\mu}_{t-1}(x) \right| \leq \beta_t^{\frac{1}{2}} \tilde{\sigma}_{t-1}(x). \quad \forall x \in \mathcal{D}, \forall t \geq 1
\end{align}
As a result, we have $\tilde{\mu}_{t-1}(x) - \beta_t^{\frac{1}{2}} \tilde{\sigma}_{t-1}(x) \leq f_t(x) \leq \tilde{\mu}_{t-1}(x) + \beta_t^{\frac{1}{2}} \tilde{\sigma}_{t-1}(x)$ $\forall x \in \mathcal{D}$ $\forall t \geq 1$. Let $x^*_t = \max_{x \in \mathcal{D}}f_t(x)$. If $x^*_t = x_t$, then the instantaneous regret $r_t = 0$. Otherwise,
\begin{align}
    f_t(x^*) - f_t(x_t) &\leq \tilde{\mu}_{t-1}(x^*_t) + \beta_t^{\frac{1}{2}} \tilde{\sigma}_{t-1}(x^*_t) - \left(\tilde{\mu}_{t-1}(x_t) - \beta_t^{\frac{1}{2}} \tilde{\sigma}_{t-1}(x_t) \right) \leq 0
\end{align}
From the definition of $x_t^*$, we known $f_t(x^*) - f_t(x_t) \geq 0$. So we have picking $x_t$ at round t produces no regret with high probability.
\end{proof}

Lemma~\ref{lem:free_no_regret} indicates that if there is no overlapping between the confidence bounds of best candidate
$x_t$ and the others, the agent can assure that $x_t$ is the optimal choice with high probability. As a consequence, if the above condition satisfies, the agent is able to receive no feedback meanwhile resulting in no performance loss at this time. This inspires us to leverage the information that exists in relative magnitudes of posterior mean and variance for the given candidates and obtain the cost-efficient query strategy, which is a loosed condition with a confidence threshold $\kappa \in (0, 1)$.

We know $\hat{y}_t(x)$ follows the Gaussian distribution $\mathcal{N}(\tilde{\mu}_{t-1}(x), \tilde{\sigma}_{t-1}(x))$. Then we know $\hat{y}_t(x_t) - \hat{y}_t(x_t^*)$ follows the distribution $$\mathcal{N}\left(\tilde{\mu}_{t-1}(x_t) - \tilde{\mu}_{t-1}(x_t^*), \sqrt{\tilde{\sigma}^2_{t-1}(x_t) + \tilde{\sigma}^2_{t-1}(x_t^*)} \right).$$
As a consequence, 
\begin{align}
1 - \Phi\left(- \frac{\tilde{\mu}_{t-1}(x_t) - \tilde{\mu}_{t-1}(x_t^*)}{\sqrt{\tilde{\sigma}^2_{t-1}(x_t) + \tilde{\sigma}^2_{t-1}(x_t^*)}}\right) \geq \gamma_t,
\end{align}
where $\Phi(\cdot)$ is the CDF of standard Gaussian distribution. To simplify further, we have 
\begin{align}
    \tilde{\mu}_{t-1}(x_t) - \tilde{\mu}_{t-1}(x_t^*) &\geq \Phi^{-1} (\gamma_t) \sqrt{\tilde{\sigma}^2_{t-1}(x_t) + \tilde{\sigma}^2_{t-1}(x_t^*)} \\
    &\geq \frac{\Phi^{-1}(\gamma_t)}{\sqrt{2}} \left(\tilde{\sigma}_{t-1}(x_t) + \tilde{\sigma}_{t-1}(x_t^*) \right).
\end{align}
If $x_t^{*} \neq x_t$, we have 
\begin{align}
    f_t(x^*) - f_t(x_t) &\leq \tilde{\mu}_{t-1}(x^*_t) + \beta_t^{\frac{1}{2}} \tilde{\sigma}_{t-1}(x^*_t) - \left(\tilde{\mu}_{t-1}(x_t) - \beta_t^{\frac{1}{2}} \tilde{\sigma}_{t-1}(x_t) \right) \\
    &\leq \tilde{\mu}_{t-1}(x^*_t) - \tilde{\sigma}_{t-1}(x^*_t) + \beta_t^{\frac{1}{2}}\left(\tilde{\sigma}_{t-1}(x^*_t) + \tilde{\sigma}_{t-1}(x_t) \right) \\
    &\leq \left(\beta_t^{\frac{1}{2}} - \frac{\Phi^{-1}(\gamma_t)}{\sqrt{2}} \right) \left( \tilde{\sigma}_{t-1}(x^*_t) + \tilde{\sigma}_{t-1}(x_t) \right).
\end{align}
Let $\gamma_t = \Phi(\sqrt{2}\beta_t)$, we recover the LCB-UCB rule (free no regret). If $\gamma_t < \Phi(\sqrt{2}\beta_t)$, then the agents suffer from a ``small regret'' $\left(\beta_t^{\frac{1}{2}} - \frac{\Phi^{-1}(\gamma_t)}{\sqrt{2}} \right) \left( \tilde{\sigma}_{t-1}(x^*_t) + \tilde{\sigma}_{t-1}(x_t) \right)$.

In practice, we found LCB-UCB rule (Lemma~\ref{lem:free_no_regret}) is a very strict rule that might has limited effect in reducing the query cost. For instance, as shown in Table~\ref{tab:synthetic_bo}, we found LCB-UCB rule is more strict than $\kappa=0.99$.

\section{Supplementary Experiments}
\label{supp:exp}
In this section, we provide the details of experimental set-up and supplementary results. 
\subsection{Evaluation on Synthetic Data} 
Apart from BO setting where the candidates are correlated (e.g., $k_\mathrm{space}$ is SE or Matérn kernels), we consider the finite-arm bandits setting where each arm is assumed independent (i.e., $k_\mathrm{space} = I$). Specifically, we consider a three-armed bandit problem where the reward for each arm is sampled from its underlying time varying functions ($\sigma^2 = 0.01$). Specifically, we design three kinds of functions: (1) a sine curve (2) a Gaussian curve, and (3) a piecewise step function, where the time horizon $T = 2000$.  Figure~\ref{fig:bandits_synthetic_data} shows an example of the time-varying functions. 
We compare with other TV bandits algorithms with Bernoulli sampling policy (Sec~\ref{sec:strategies}), i.e., $m \sim \Ber(0.1)$ incorporated, including EXP3.S~\citep{auer2002nonstochastic}, $\epsilon$-greedy~\citep{kuleshov2000algorithms}, Softmax (Boltzmann Exploration)~\citep{kuleshov2000algorithms}, UCB~\citep{auer2002finite} and GP-UCB~\citep{Srinivas09}. We used the RBF kernel in GP models and the parameters (e.g., length scale) for GP are optimized using maximum likelihood. Each experiment is repeated 100 times with different random seeds.

\begin{figure}[t]
\centering
\begin{subfigure}[b]{.325\textwidth}
    \centering
    \includegraphics[width=\textwidth]{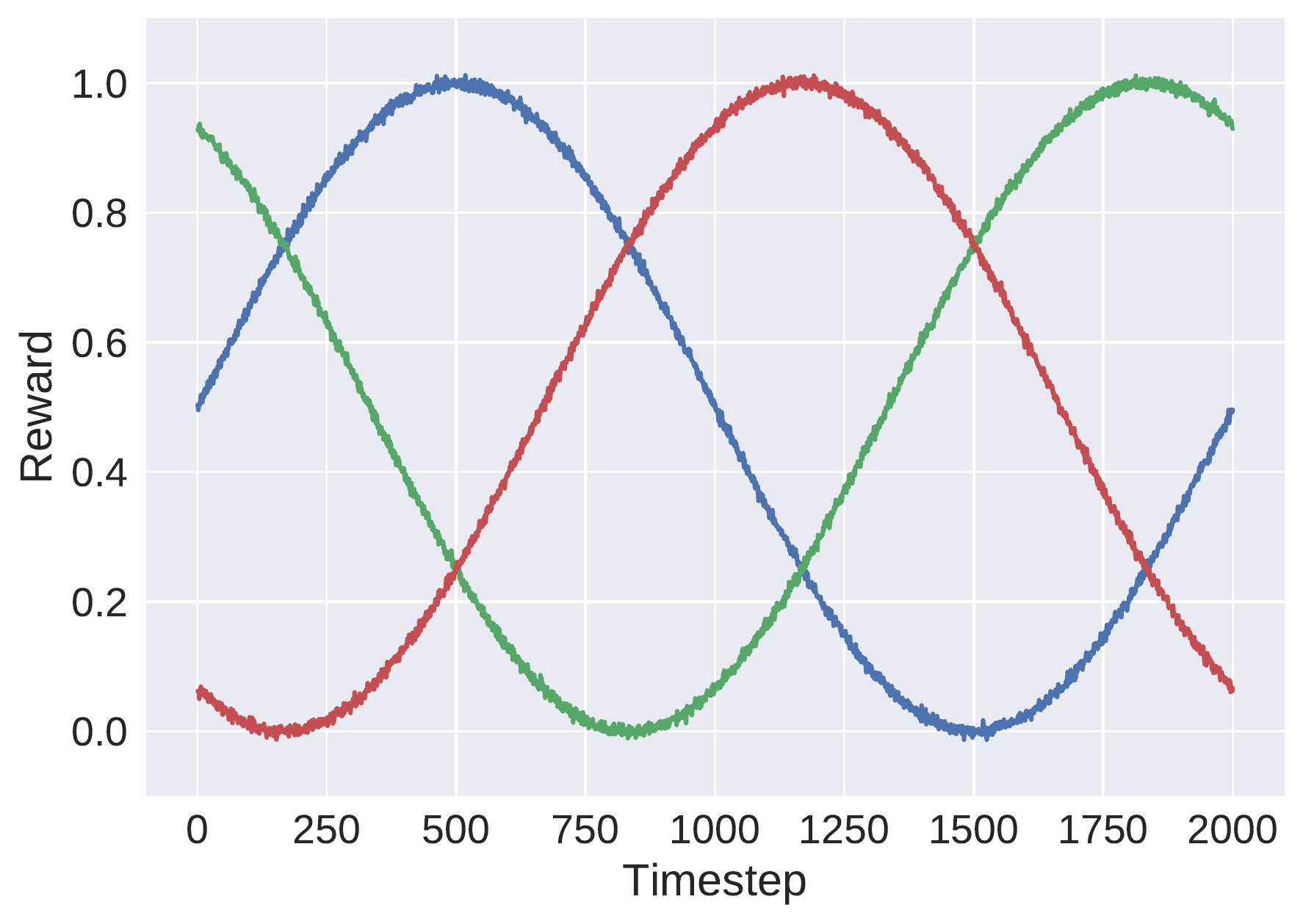}
    \caption{\texttt{Sine}}
\end{subfigure}
\begin{subfigure}[b]{.325\textwidth}
    \centering
    \includegraphics[width=\textwidth]{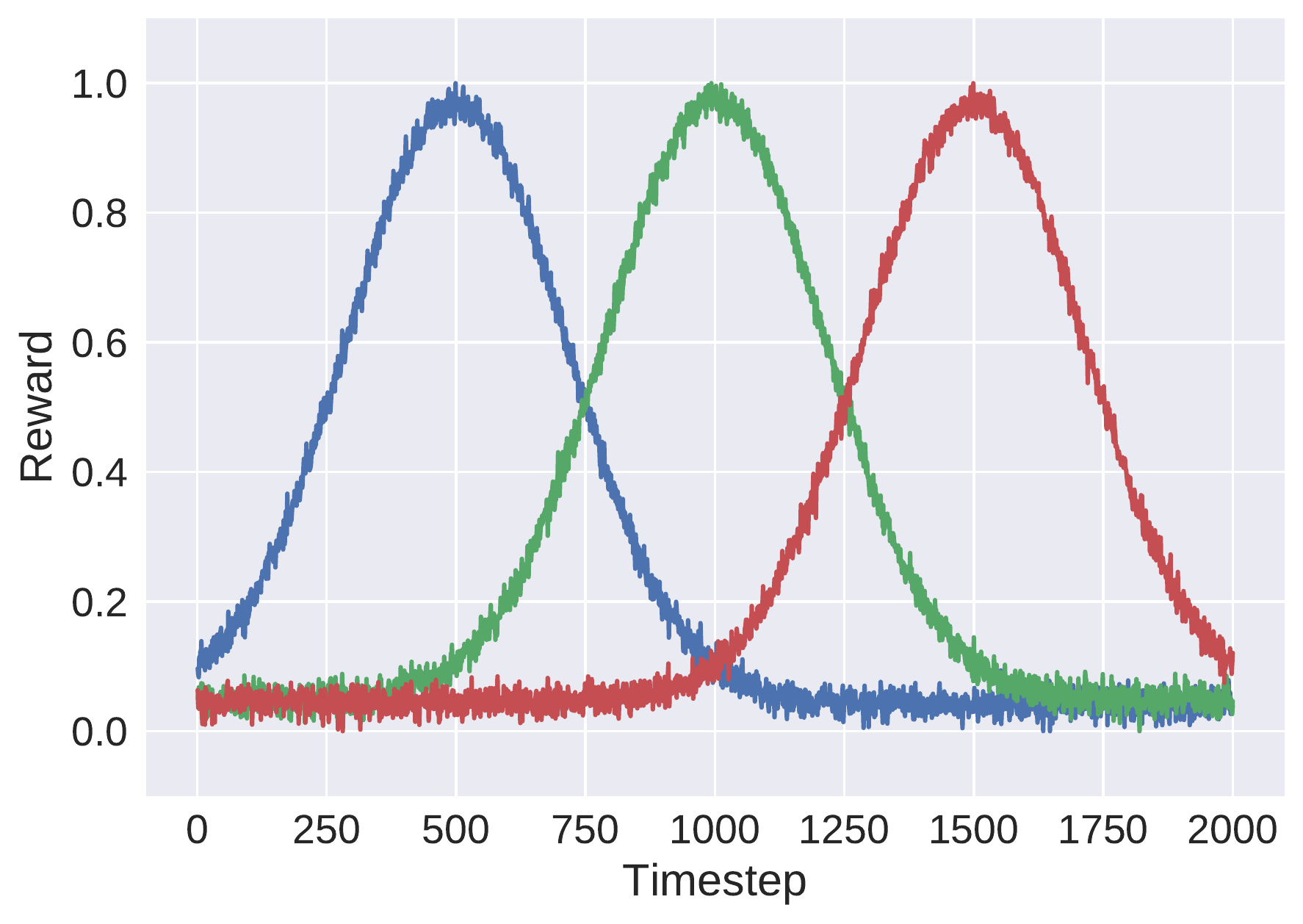}
    \caption{\texttt{Gaussian}}
\end{subfigure}
\begin{subfigure}[b]{.325\textwidth}
    \centering
    \includegraphics[width=\textwidth]{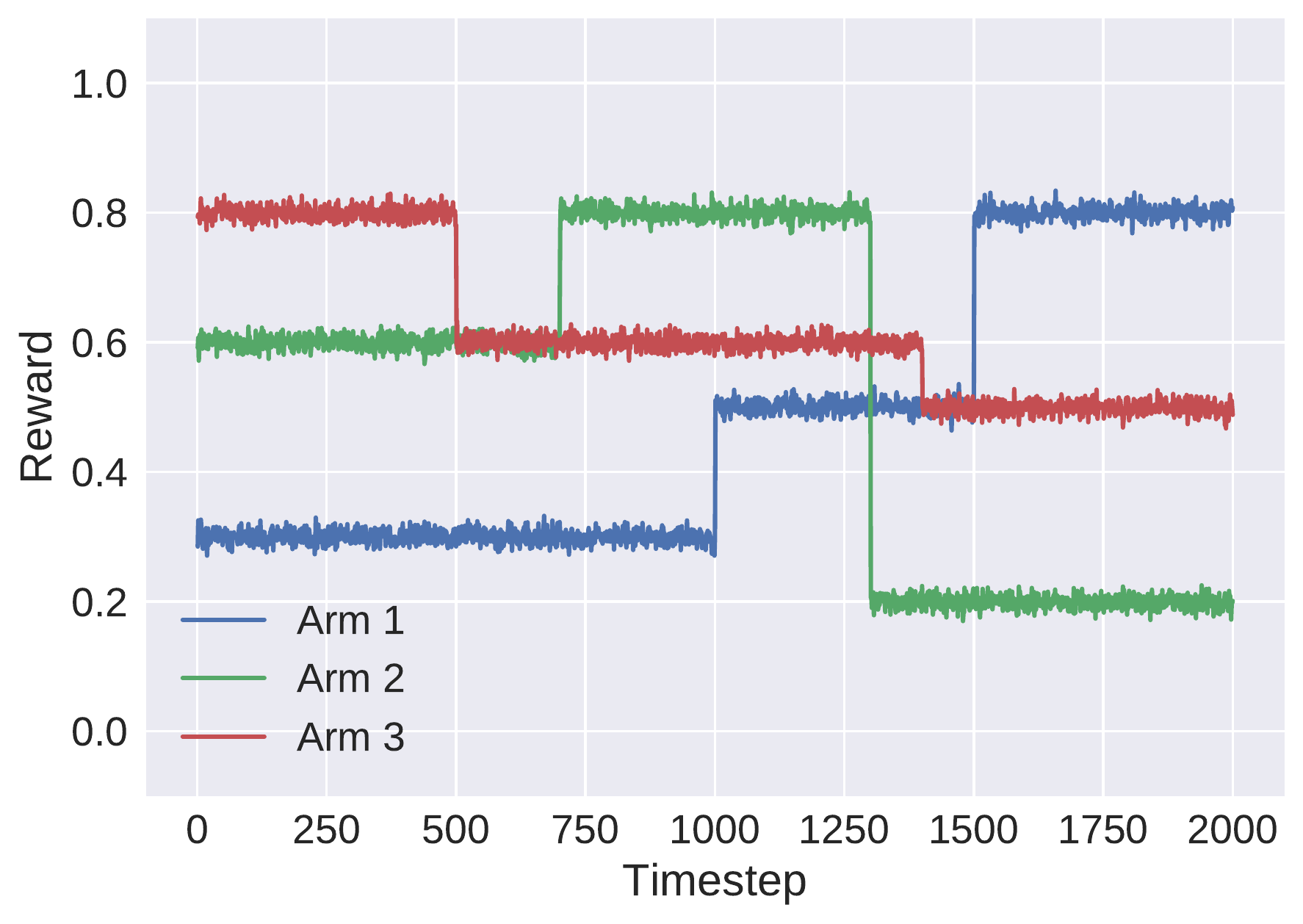}
    \caption{\texttt{Piecewise}}
\end{subfigure}
\caption{Time-varying functions considered for the synthetic bandit experiments.} 
\label{fig:bandits_synthetic_data}
\end{figure}

\begin{table}
\centering
\caption{Comparison of CE-GP-UCB ($\kappa= 95\%$) with standard TV bandits algorithms with Bernoulli query strategy on three kinds of synthetic data.}
\label{tab:tv_bandits}
\resizebox{\textwidth}{!}{
\begin{tabular}{@{}lcccccc@{}}
\toprule
 & \multicolumn{2}{c}{Sine} & \multicolumn{2}{c}{Gaussian} & \multicolumn{2}{c}{Piecewise} \\ \cmidrule(l){2-7} 
 & $R_T / T$ & $C_T$ & $R_T / T$ & $C_T$ & $R_T / T$ & $C_T$ \\ \midrule
EXP3.S~\citep{auer2002nonstochastic} & 0.384 $\pm$ 0.036 & 201 $\pm$ 15 & 0.257 $\pm$ 0.017 & 200 $\pm$ 11 & 0.437 $\pm$ 0.046 & 201 $\pm$ 14 \\
$\epsilon$-greedy~\citep{kuleshov2000algorithms} & 0.167 $\pm$ 0.041 & 198 $\pm$ 14  & 0.145 $\pm$ 0.026 & 200 $\pm$ 12 &  0.196 $\pm$ 0.026 & 199 $\pm$ 14 \\
Softmax~\citep{kuleshov2000algorithms} & 0.132 $\pm$ 0.025 & 199 $\pm$ 13 & 0.159 $\pm$ 0.032 & 201 $\pm$ 14 & 0.256 $\pm$ 0.037 & 198 $\pm$ 14 \\
UCB~\citep{auer2002finite} & 0.084 $\pm$ 0.010 & 200 $\pm$ 14 & 0.112 $\pm$ 0.026 & 201 $\pm$ 15 & 0.172 $\pm$ 0.042 &  202 $\pm$ 14\\
GP-UCB~\citep{Srinivas09} & 0.104 $\pm$ 0.012 & 200 $\pm$ 14 & 0.139 $\pm$ 0.025 & 200 $\pm$ 14 & 0.122 $\pm$ 0.009 & 201 $\pm$ 12 \\
\textbf{CE-GP-UCB} & \textbf{0.017} $\boldsymbol{\pm}$ \textbf{0.009} & \textbf{52} $\boldsymbol{\pm}$ \textbf{7} & \textbf{0.038} $\boldsymbol{\pm}$ \textbf{0.017} & \textbf{53} $\pm$ \textbf{20} & \textbf{0.028} $\boldsymbol{\pm}$ \textbf{0.002} & \textbf{52} $\boldsymbol{\pm}$ \textbf{5} \\ \bottomrule
\end{tabular}
}
\end{table}

\begin{figure}[t]
\centering
\begin{subfigure}[b]{.325\textwidth}
    \centering
    \includegraphics[width=\textwidth]{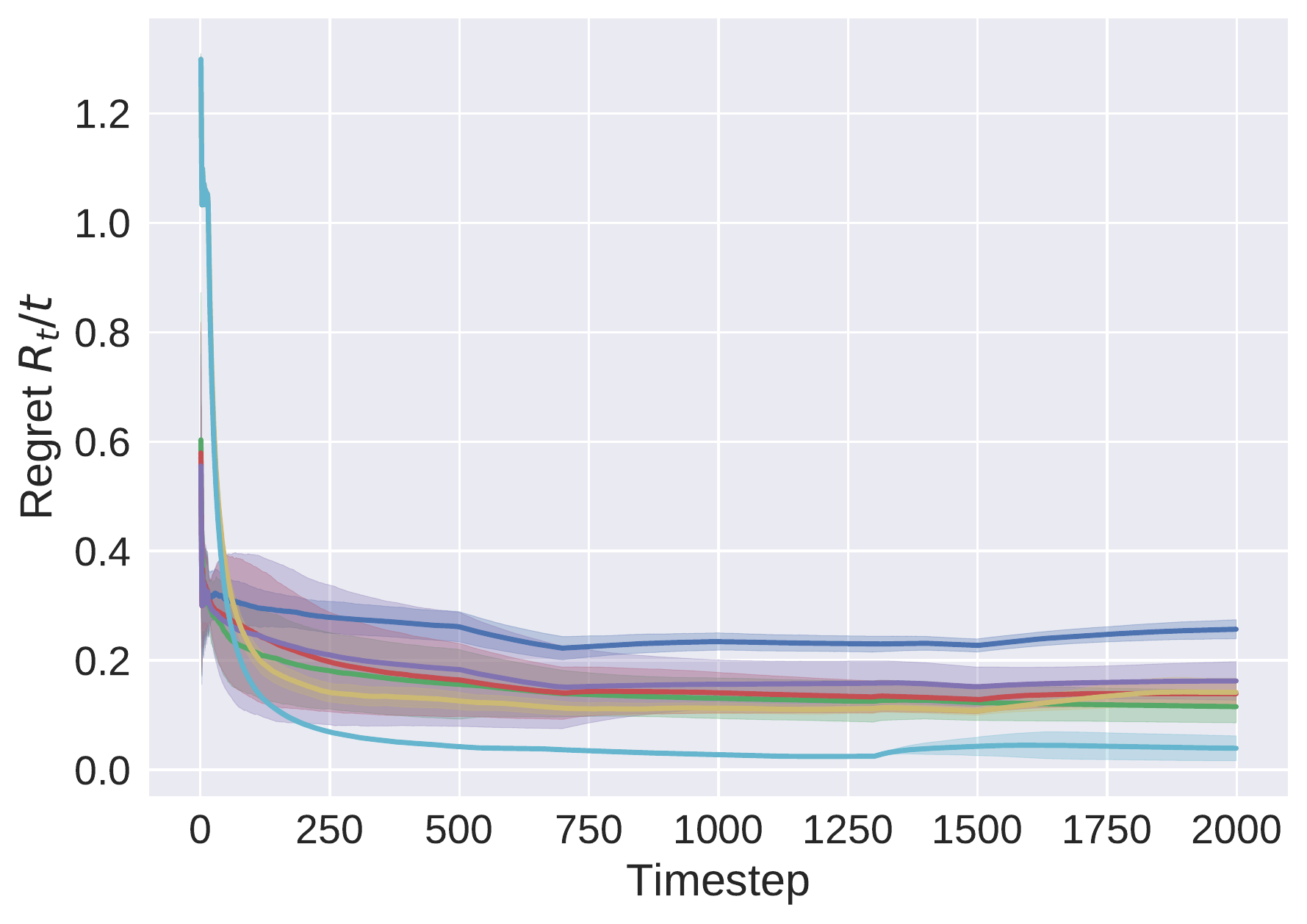}
    \vspace{-5mm}
    \caption{\texttt{Sine}}
    \label{fig:bandits_sine}
\end{subfigure}
\begin{subfigure}[b]{.325\textwidth}
    \centering
    \includegraphics[width=\textwidth]{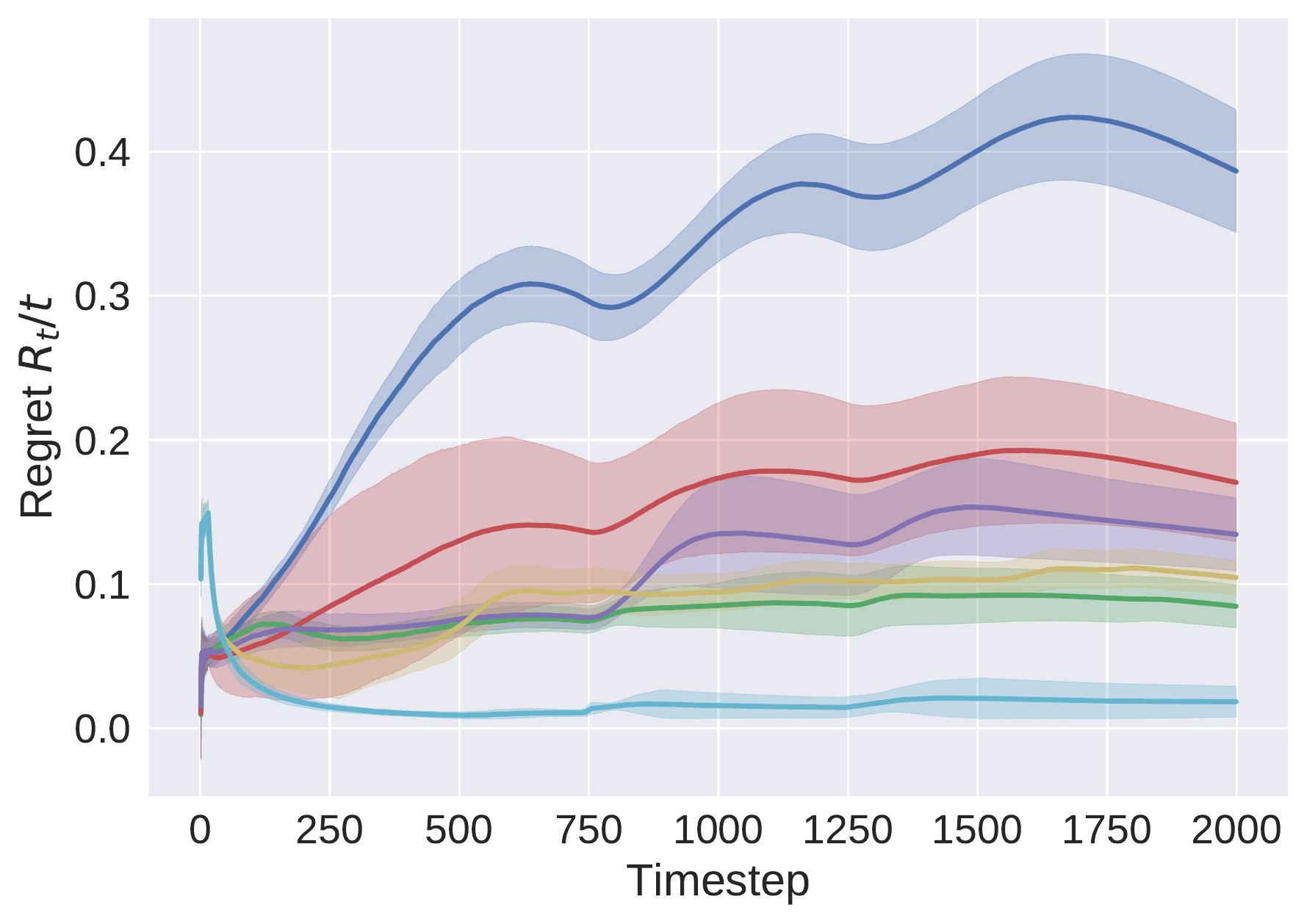}
    \vspace{-5mm}
    \caption{\texttt{Gaussian}}
    \label{fig:bandits_gaussian}
\end{subfigure}
\begin{subfigure}[b]{.325\textwidth}
    \centering
    \includegraphics[width=\textwidth]{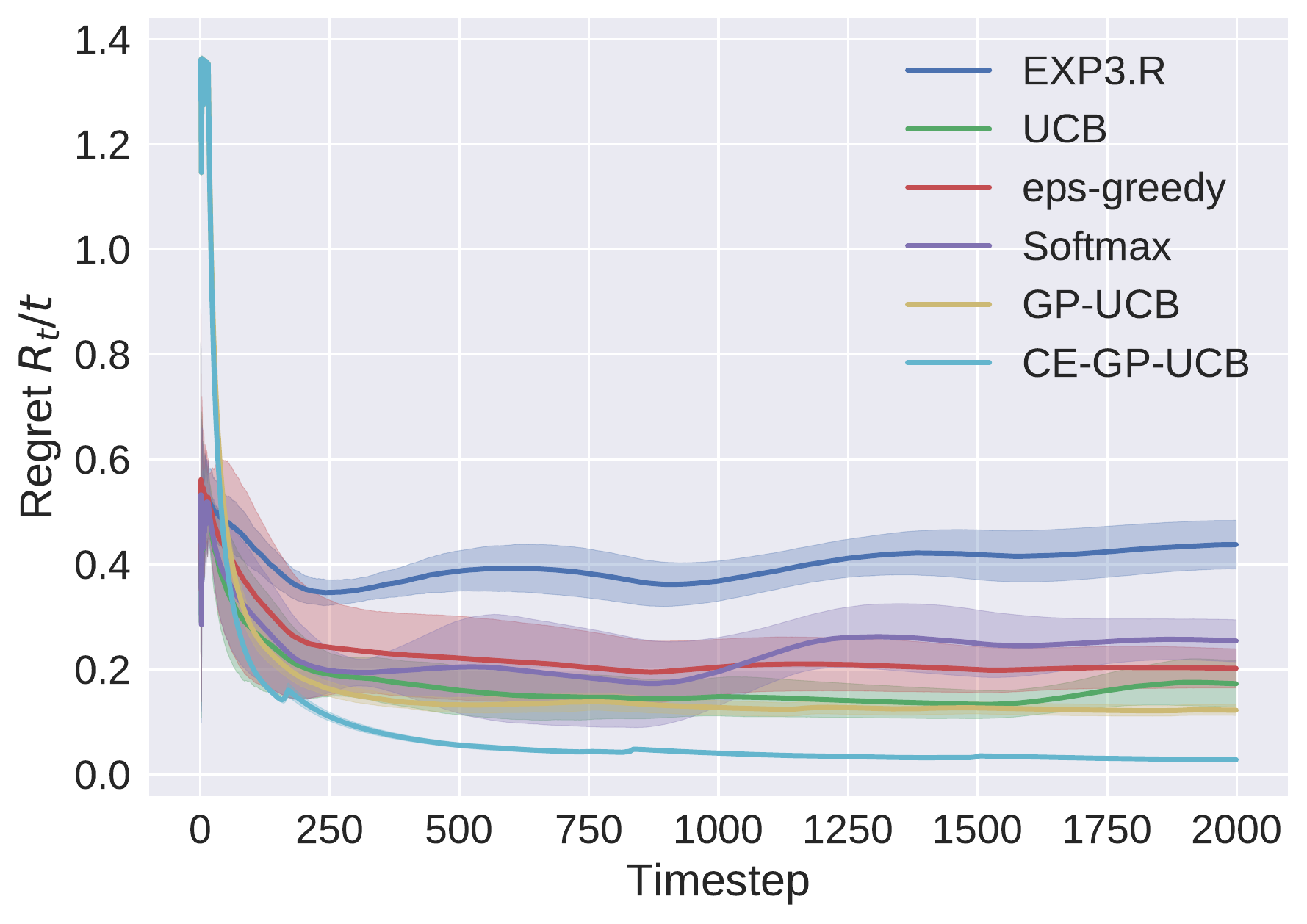}
    \vspace{-5mm}
    \caption{\texttt{Piecewise}}
    \label{fig:bandits_piecewise}
\end{subfigure}
\vspace{-0.05in}
\caption{Average regret $R_T/T$ for time-varying bandits algorithms on three different synthetic data.} %
\vspace{-0.12in}
\label{fig:bandits_synthetic}
\end{figure}

As shown in Table~\ref{tab:tv_bandits}, we found proposed CE-GP-UCB consistently achieves lower regrets while interacting less with the unknown functions. This is because CE-GP-UCB automatically assesses the confidence of current policy based on current GP models and given problems, and skips the queries if it is pretty confident the decision is correct. Moreover, compared to GP-UCB with full observations, our method saves around 97.5\% queries, which results in $40\times$ computation cost saving in evaluating the unknown objective. 
Figure~\ref{fig:bandits_synthetic} shows how $R_t/t$ evolves as training proceeds for different algorithms. It again verifies cost-efficient query rule can adapt to different problems quickly and achieve lower regrets.

\paragraph{TV Bayesian optimization}

Table~\ref{tab:synthetic_bo} presents numerical performance of CE-GP-UCB on synthetic data with different confidence threshold $\kappa$. The performance is averaged over 50 independent trials. We found that CE-GP-UCB can almost recover the performance of original TV-GP-UCB meanwhile significantly cutting off the cost for interacting with unknown functions which is usually very expensive. Taking $\epsilon = 0.05$ for instance, CE-GP-UCB ($\kappa=0.9$) reduces up to 40\% queries meanwhile only suffers from 2.6\% regret loss. %

\begin{table}[htbp!]
\caption{Numerical performance on synthetic data. $R_T/T$ and $C_T$ denote average regret and total cost for observing rewards. @$\kappa$ denotes the confidence threshold for CE-GP-UCB is $\kappa$. @* denotes the LCB-UCB rule (Lemma~\ref{lem:free_no_regret}). We highlight the setting within 10\% regret loss compared to TV-GP-UCB with full observations. The performance is averaged over 50 independent trials. %
}
\label{tab:synthetic_bo}
\resizebox{\textwidth}{!}{
\begin{tabular}{@{}lcccccccccc@{}}
\toprule
 & \multicolumn{2}{c}{$\epsilon=0.003$} & \multicolumn{2}{c}{$\epsilon=0.005$} & \multicolumn{2}{c}{$\epsilon=0.01$} & \multicolumn{2}{c}{$\epsilon=0.03$} & \multicolumn{2}{c}{$\epsilon=0.05$} \\ \cmidrule(l){2-11} 
 & $R_T/T$ & $C_T$ & $R_T/T$ & $C_T$ & $R_T/T$ & $C_T$ & $R_T/T$ & $C_T$ & $R_T/T$ & $C_T$ \\ \midrule
R-GP-UCB & $0.352\pm0.187$ & $499\pm0$ & $0.371\pm0.166$ & $499\pm0$ & $0.425\pm0.118$ & $499\pm0$ & $0.516\pm0.093$ & $499\pm0$ & $0.581\pm0.093$ & $499\pm0$ \\
TV-GP-UCB & $0.094\pm0.035$ & $499\pm0$ & $0.126\pm0.044$ & $499\pm0$ & $0.184\pm0.056$ & $499\pm0$ & $0.305\pm0.034$ & $499\pm0$ & $0.392\pm0.034$ & $499\pm0$ \\
\midrule
TV-GP-UCB Ber(0.2) & $0.231\pm0.089$ & $99\pm8$ & $0.275\pm0.084$ & $99\pm9$ & $0.350\pm0.094$ & $97\pm9$ & $0.561\pm0.122$ & $100\pm9$ & $0.694\pm0.090$ & $100\pm8$ \\
TV-GP-UCB Ber(0.3) & $0.176\pm0.093$ & $147\pm9$ & $0.218\pm0.104$ & $150\pm12$ & $0.312\pm0.089$ & $144\pm10$ & $0.485\pm0.089$ & $150\pm10$ & $0.592\pm0.096$ & $153\pm10$ \\
TV-GP-UCB Ber(0.4) & $0.147\pm0.072$ & $200\pm10$ & $0.186\pm0.073$ & $200\pm10$ & $0.275\pm0.090$ & $201\pm11$ & $0.428\pm0.073$ & $203\pm11$ & $0.522\pm0.073$ & $203\pm9$ \\
TV-GP-UCB Ber(0.5) & $0.121\pm0.060$ & $247\pm10$ & $0.160\pm0.061$ & $249\pm12$ & $0.241\pm0.073$ & $252\pm12$ & $0.382\pm0.064$ & $249\pm12$ & $0.480\pm0.058$ & $250\pm12$ \\
TV-GP-UCB Ber(0.6) & $0.119\pm0.055$ & $296\pm9$ & $0.144\pm0.041$ & $300\pm11$ & $0.230\pm0.078$ & $298\pm11$ & $0.363\pm0.045$ & $298\pm11$ & $0.452\pm0.045$ & $298\pm12$ \\
TV-GP-UCB Ber(0.7) & $0.112\pm0.049$ & $348\pm10$ & $0.151\pm0.060$ & $348\pm11$ & $0.209\pm0.068$ & $348\pm10$ & $0.335\pm0.038$ & $350\pm11$ & $0.429\pm0.040$ & $350\pm9$ \\
TV-GP-UCB Ber(0.8) & $0.102\pm0.046$ & $401\pm9$ & $0.140\pm0.058$ & $399\pm11$ & $0.198\pm0.067$ & $399\pm9$ & $0.328\pm0.035$ & $400\pm8$ & $0.415\pm0.037$ & $400\pm7$ \\
TV-GP-UCB Ber(0.9) & $0.098\pm0.048$ & $450\pm7$ & $0.132\pm0.051$ & $449\pm7$ & $0.196\pm0.066$ & $448\pm7$ & $0.311\pm0.036$ & $448\pm6$ & $0.404\pm0.035$ & $449\pm7$ \\
\midrule
CE-GP-UCB @0.60 & $0.535\pm0.416$ & $8\pm5$ & $0.558\pm0.319$ & $12\pm7$ & $0.600\pm0.289$ & $20\pm11$ & $0.642\pm0.141$ & $53\pm11$ & $0.680\pm0.119$ & $85\pm17$ \\
CE-GP-UCB @0.70 & $0.410\pm0.329$ & $16\pm11$ & $0.435\pm0.253$ & $21\pm13$ & $0.442\pm0.186$ & $41\pm25$ & $0.504\pm0.096$ & $91\pm21$ & $0.533\pm0.071$ & $134\pm21$ \\
CE-GP-UCB @0.75 & $0.341\pm0.249$ & $22\pm17$ & $0.367\pm0.199$ & $34\pm20$ & $0.368\pm0.176$ & $62\pm27$ & $0.441\pm0.077$ & $117\pm25$ & $0.484\pm0.074$ & $167\pm30$ \\
CE-GP-UCB @0.80 & $0.286\pm0.169$ & $31\pm27$ & $0.297\pm0.157$ & $51\pm36$ & $0.304\pm0.117$ & $76\pm27$ & $0.387\pm0.065$ & $146\pm31$ & $0.451\pm0.053$ & $200\pm30$ \\
CE-GP-UCB @0.85 & $0.226\pm0.131$ & $51\pm41$ & $0.248\pm0.095$ & $78\pm42$ & $0.266\pm0.069$ & $101\pm46$ & $0.356\pm0.054$ & $180\pm34$ & $0.416\pm0.041$ & $235\pm35$ \\
CE-GP-UCB @0.90 & $0.188\pm0.089$ & $82\pm73$ & $0.202\pm0.080$ & $101\pm63$ & $0.238\pm0.070$ & $140\pm48$ & $0.322\pm0.045$ & $233\pm37$ & $0.400\pm0.036$ & $291\pm30$ \\
CE-GP-UCB @0.95 & $0.157\pm0.074$ & $125\pm77$ & $0.161\pm0.051$ & $163\pm73$ & $0.210\pm0.052$ & $207\pm60$ & $0.307\pm0.037$ & $310\pm42$ & $0.397\pm0.034$ & $371\pm29$ \\
CE-GP-UCB @0.99 & $0.112\pm0.037$ & $292\pm104$ & $0.128\pm0.030$ & $309\pm83$ & $0.176\pm0.031$ & $363\pm60$ & $0.295\pm0.031$ & $435\pm26$ & $0.386\pm0.033$ & $468\pm15$ \\
CE-GP-UCB @* & $0.103\pm0.034$ & $386\pm79$ & $0.126\pm0.031$ & $414\pm50$ & $0.168\pm0.033$ & $442\pm41$ & $0.292\pm0.031$ & $482\pm11$ & $0.384\pm0.031$ & $492\pm7$ \\
\bottomrule
\end{tabular}
}
\end{table}

\subsection{Self-Tuning Networks}

\begin{figure}[t]
\centering
\begin{subfigure}[b]{.41\textwidth}
    \centering
    \includegraphics[width=\textwidth]{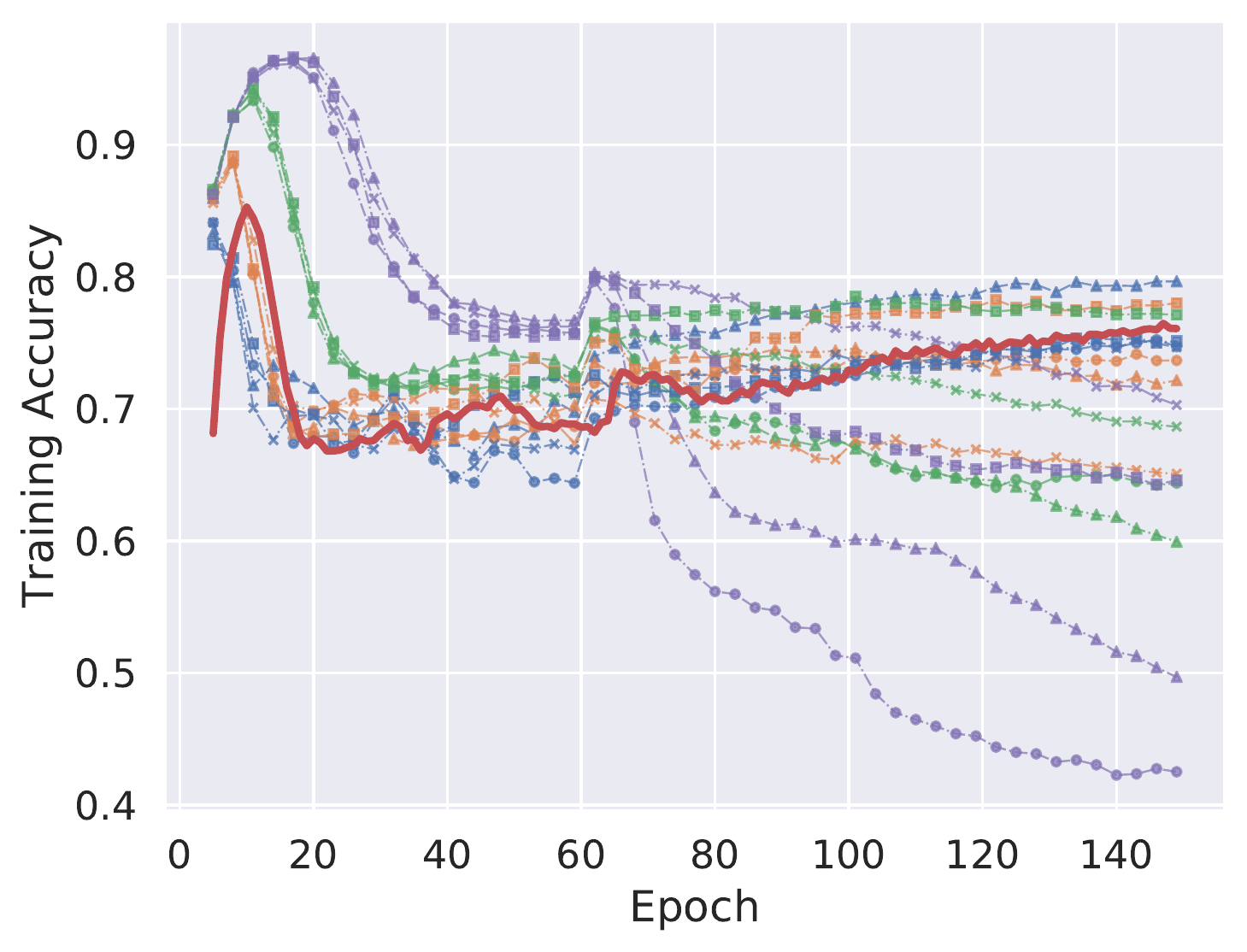}
    \vspace{-5mm}
    \caption{Training Accuracy}
    \label{fig:stn_train_acc}
\end{subfigure}
\begin{subfigure}[b]{.58\textwidth}
    \centering
    \includegraphics[width=\textwidth]{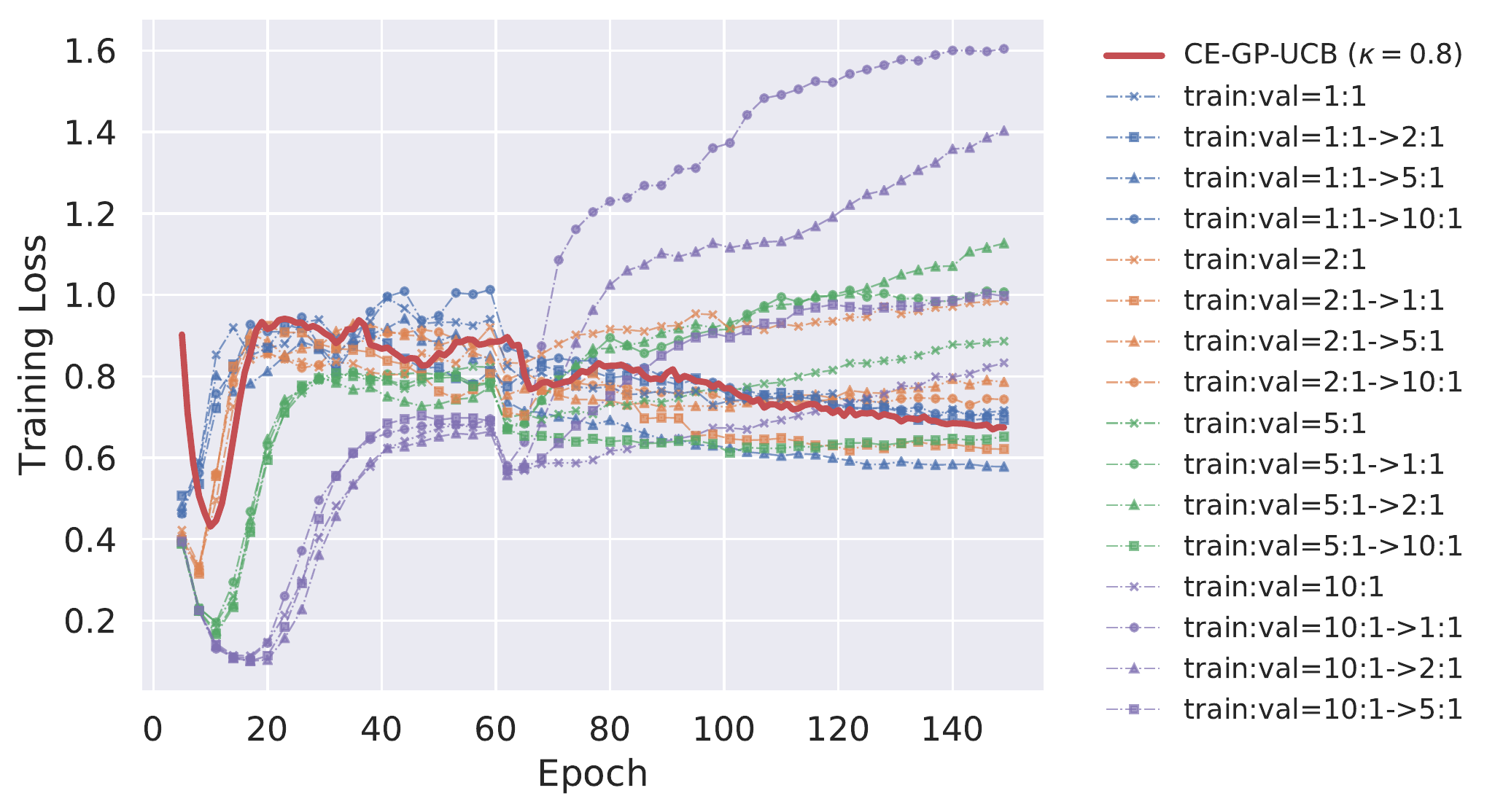}
    \vspace{-5mm}
    \caption{Training Loss}
    \label{fig:stn_train_loss}
\end{subfigure}
\caption{Learning curves on (a) training accuracy \& (b) training loss under different tuning schedules. Note that the training curves for STN is different form standard training since at the late stage of training, large amounts of data augmentations are added to avoid overfitting. Our proposed CE-GP-UCB finds an unique pattern that differs from static tuning schedules or switching schedules automatically.}
\label{fig:stn_train}
\end{figure}

\begin{figure}[htbp!]
\centering
\begin{subfigure}[b]{.325\textwidth}
    \centering
    \includegraphics[width=\textwidth]{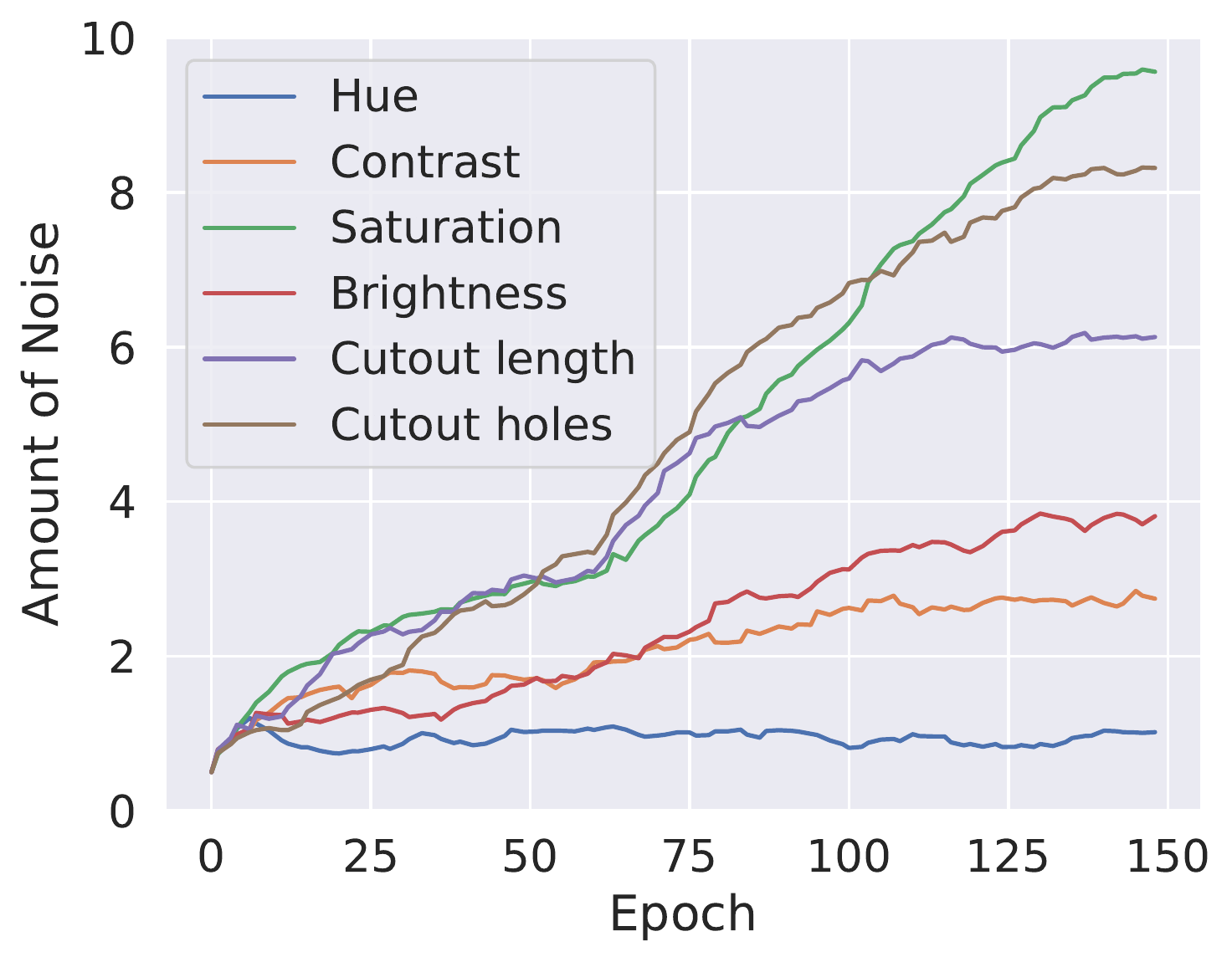}
    \vspace{-5mm}
    \caption{train/val=1:1}
    \label{fig:stn_aug_standard}
\end{subfigure}
\begin{subfigure}[b]{.325\textwidth}
    \centering
    \includegraphics[width=\textwidth]{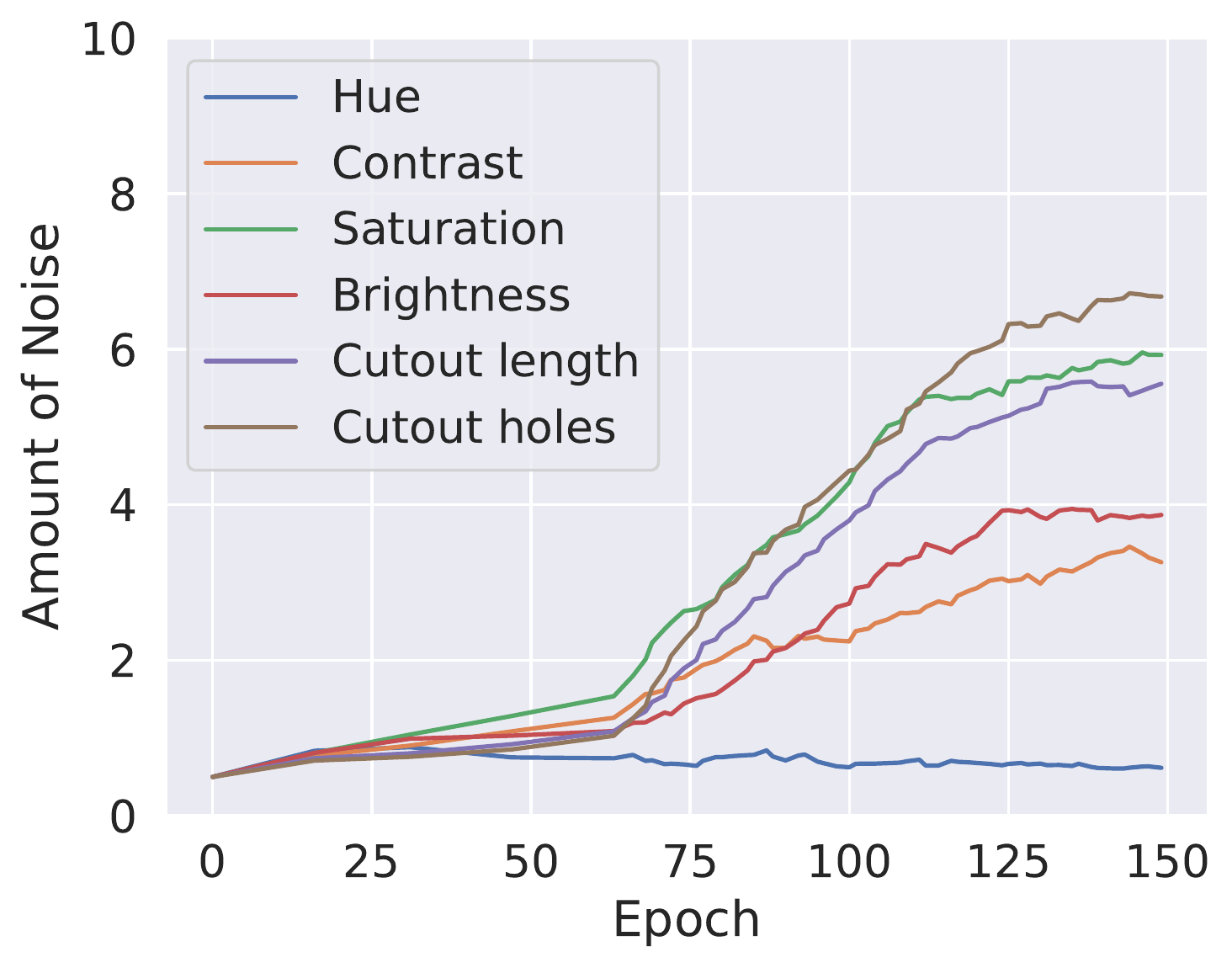}
    \vspace{-5mm}
    \caption{train/val=10:1->1:1}
    \label{fig:stn_aug_worst}
\end{subfigure}
\begin{subfigure}[b]{.325\textwidth}
    \centering
    \includegraphics[width=\textwidth]{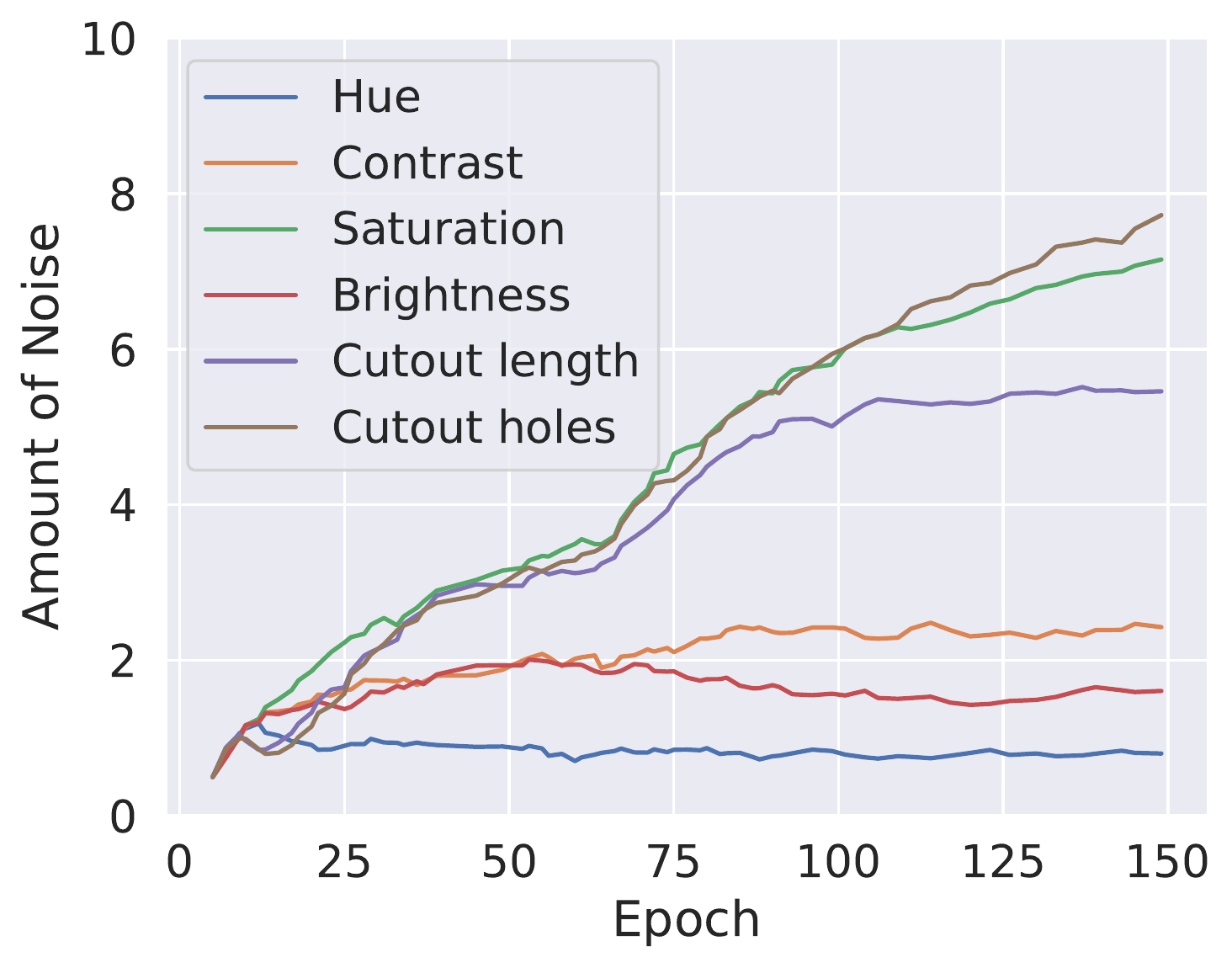}
    \vspace{-5mm}
    \caption{CE-GP-UCB (ours)}
    \label{fig:stn_aug_ours}
\end{subfigure}
\caption{The hyperparameter schedule prescribed by the STN (VGG16 on CIFAR-10).}
\label{fig:stn_augs}
\end{figure}

Figure~\ref{fig:stn_train} shows the training curves (loss and accuracy) under different tuning schedules. We observed that the training curves of STN differ from standard training curves since large amounts of data augmentations are added at the late stage of training which results in a huge performance loss on training data. However, even though the performance is significantly worse on the training set (e.g., 60\% to 80\% training accuracy), the model still performs well on the clean images without augmentations on the validation set and test set (see Figure~\ref{fig:stn_learning_curves}).

Figure~\ref{fig:stn_augs} presents the hyperparameter schedule prescribed by the STN under different tuning schedules. In general, the amount of noise added to the image increases as the training proceeds to alleviate the effects of overfitting. However, we found that CE-GP-UCB results in a slightly different pattern. Specifically, STN decreases the brightness at the late training stage with CE-GP-UCB, however, it increases the brightness as other data augmentations under static or switching tuning schedules.

\subsection{Unsupervised Contrastive Representation Learning}
\paragraph{Implementation Details for SimCLR} Following~\cite{Chen20simclr}, we take ResNet-50~\citep{resnet} as the encoder network, and a 2-layer MLP projection head to project the representation. We trained SimCLR with 64 GPUs on ImageNet100~\citep{Tian19cmc} (100 classes of ImageNet) for 200 epochs and set the batch-size as $64\times56=3584$. To stabilize the training with large batch size, the LARS optimizer is adopted with the learning rate $4.8$. For linear readout evaluation, a linear layer is trained from scratch for 10 epochs. We adopt SGD optimizer with a learning rate of 10 and momentum 0.9 following~\cite{Tian19cmc}.

\paragraph{Implementation Details for CE-GP-UCB}
We integrate CE-GP-UCB algorithm into the popular \texttt{botorch} library~\citep{botorch} for higher efficiency and stable performance. Specifically, we use Matern5/2 and set the forgetting rate as $\epsilon=0.01$. To find all the local optima, we randomly initialize 50 locations and use L-BFGS algorithm~\citep{LiuN89} to find the local optima. Furthermore, we use meanshift~\citep{Cheng95} algorithm to suppress close candidates with a bandwidth of 0.2.
To avoid catastrophic performance loss and misleading signals by some extremely biased decision, we clip the reward signals between $[-2, 2]$ and set the minimal probability to apply the cropping as 0.5. 

\end{document}